\documentclass[11pt]{article}
\usepackage[authoryear,round]{natbib}
\usepackage[utf8]{inputenc}
\usepackage[T1]{fontenc}
\usepackage{amsmath,amssymb,amsthm,amsfonts,mathtools}
\usepackage{array,siunitx,graphicx}
\usepackage{authblk}
\usepackage{enumerate}
\usepackage{caption}
\usepackage{titlesec}
\usepackage{newtxtext}
\usepackage[subscriptcorrection]{newtxmath}
\usepackage[margin=1in]{geometry}
\usepackage{url}
\usepackage{booktabs}
\usepackage[colorlinks=true,linkcolor=blue,citecolor=blue,urlcolor=blue]{hyperref}
\hypersetup{
  pdftitle={Multitask Regression with Pairwise Fusion},
  pdfauthor={Xiaodong Li and Zhentao Li}
}

\titleformat{\section}[hang]
  {\normalfont\bfseries\fontsize{11}{13pt plus .8pt minus .6pt}\selectfont}
  {\thesection.}{1em}{}
\titleformat{\subsection}[hang]
  {\normalfont\bfseries\fontsize{11}{13pt plus .8pt minus .6pt}\selectfont}
  {\thesubsection}{1em}{}

\allowdisplaybreaks[3]
\AddToHook{cmd/maketitle/after}{%
  \setlength{\abovedisplayskip}{7.5pt plus 2pt minus 2pt}%
  \setlength{\belowdisplayskip}{7.5pt plus 2pt minus 2pt}%
  \setlength{\abovedisplayshortskip}{2pt plus 2pt}%
  \setlength{\belowdisplayshortskip}{5pt plus 2pt minus 2pt}%
}
\AddToHook{env/thebibliography/before}{\normalsize}
\AddToHook{cmd/appendix/before}{\normalsize}

\newcommand{\predictionerror}{%
\frac1{2T}\sum_{t=1}^T\frac{\|X^{(t)}h^{(t)}\|_2^2}{n_t}}

\theoremstyle{plain} \numberwithin{equation}{section}
\newtheorem{theorem}{Theorem}[section]

\newtheorem{lemma}[theorem]{Lemma}
\theoremstyle{definition}
\newtheorem{assumption}{Assumption}[section]
\theoremstyle{remark}
\newtheorem{remark}{Remark}[section]

\title{Multitask Regression with Pairwise Fusion}
\author[1]{Xiaodong Li}
\author[1]{Zhentao Li}
\affil[1]{Department of Statistics, University of California, Davis, Davis, CA 95616, USA}
\date{}

\begin{document}
\maketitle

\begin{abstract}
We study multitask regression when coefficient sharing can differ by predictor.
For a given predictor, many tasks may have the same coefficient while a few
differ, and the exceptional tasks need not be the same for another predictor.
We describe this structure by two quantities: the number of active predictors
and the total number of task coefficients that differ from the most common
value for their predictor.  We estimate the coefficient matrix by penalizing
all pairwise coefficient differences across tasks, with an additional group
penalty when predictor selection is needed.  The resulting upper and lower
bounds have the same dependence on these two quantities.  We also consider the
stronger setting in which a large set of tasks shares one entire coefficient vector.
Under explicit sample-size conditions, the same pairwise estimator pools those
tasks exactly, while allowing the remaining tasks to differ.  Simulations and
household energy data illustrate the transition between broad sharing and
task-specific coefficients.
\end{abstract}

\section{INTRODUCTION}

Multitask regression estimates several related regression models jointly so
they can share information.  The main practical question is how much they
should share.  Fitting every task separately can be noisy when each task has
limited data, while forcing all tasks to use one coefficient vector can be
badly biased when their coefficients truly differ.  Many applications lie
between these extremes.

We focus on a simple form of partial sharing.  For a given predictor, many
tasks may have exactly the same coefficient and only a few may differ.  The
exceptional tasks can change from predictor to predictor, so the model does not
require one task clustering to work for every predictor.  In the household
energy example below, for instance, one housing characteristic may have a
similar association with electricity use in most states, while another may
differ in a different subset of states.

Two numbers summarize the resulting structure.  Let \(s\) denote the number of
active predictors, and let \(r\) denote the total number of task coefficients
that differ from the most common coefficient for their predictor.  Thus \(s\)
measures how many predictors matter, while \(r\) measures how many coefficients
depart from predictor-specific shared values.  Section~\ref{sec:predictor-specific-sharing}
gives the formal definitions.

We estimate this structure by penalizing pairwise coefficient differences
across tasks, separately for each predictor.  The method needs no task ordering
or external similarity information.  We study pairwise fusion by itself and in
combination with a group penalty that selects predictors.

To state the main rate informally, suppose there are \(p\) predictors, \(T\)
tasks, and \(n\) observations per task, with noise level \(\sigma\).  In the
balanced design case, the sparse estimator has squared Frobenius error of order
\[
 \frac{\sigma^2}{n}
 \left\{s\log(ep/s)+r\log(esT/r)\right\}
\]
with high probability.  The first term is the usual cost of finding the active
predictors.  The second is the cost of finding and estimating the exceptional
task coefficients.  An expected-risk minimax lower bound has the same
dependence on these quantities.  When coefficients are allowed to vary freely
across tasks, the second term becomes of order \(sT\), recovering the usual
group-sparse multitask scale.

We also consider the stronger case in which a large subset of tasks shares the
same full coefficient vector.  Under explicit sample-size conditions, the
pairwise estimator identifies and pools those tasks exactly without knowing the
subset in advance.  The error for the shared coefficient is the usual pooled
estimation error plus a contribution from the remaining tasks.

These results describe two forms of borrowing strength.  The first allows the
sharing pattern to change from predictor to predictor and quantifies the gain
from repeated coefficients.  The second shows when a genuinely shared set of
tasks can be pooled automatically.  The simulations examine both settings,
and the household energy example shows that broad sharing and more dispersed
coefficients can coexist across predictors.

\paragraph{Sparse multitask regression.}
Group penalties are a standard way to select the same predictors across tasks
\citep{yuan2006model,argyriou2008convex,obozinski2010joint,
lounici2009taking,lounici2011oracle}.  We retain this shared-support idea while
also allowing many coefficients of an active predictor to be equal across
tasks.

\paragraph{Coefficient fusion.}
Pairwise coefficient penalties have been used for structured estimation and
task clustering
\citep{tibshirani2005sparsity,hutter2016optimal,sadhanala2016tv,
bellec2017graphslope,jacob2009clustered,kim2009graphguided,
chen2010graphstructured,hallac2015network,she2010sparse,tang2016fused,
ma2017concave,dondelinger2018high,liu2019fused,okazaki2024multi}.
The approach of \citet{tang2016fused} is especially close because it fuses
coefficients separately for each covariate.  Our focus is the case with no prespecified task relationships and on rates determined by the number of predictor-specific
exceptions.

\paragraph{Rates.}
Classical sparse and group-sparse lower bounds account for the unknown active
predictors \citep{raskutti2011minimax,ye2010rate,lounici2011oracle}.  Here one
must also locate the exceptional task coefficients within those predictors,
which produces the additional \(r\log(esT/r)\) term.

Sections~\ref{sec:model-methods}--\ref{sec:shared-task-taskwise} develop the
model, estimation rates, and shared-task pooling result.
Section~\ref{sec:computation} describes computation,
Sections~\ref{sec:simulation}--\ref{sec:real-data-recs} report numerical
results, and the appendices contain proofs and experimental details.

\section{MODEL AND METHODS}
\label{sec:model-methods}

The coefficient matrix can be simple in two ways.  Some predictors may be zero
in every task, and some active predictors may have the same coefficient in
many tasks.  We use \(s\) for the number of active predictors and \(r\) for the
total number of exceptions from predictor-specific shared values.  We use one
penalty for each feature.  Throughout, \(T\ge2\).

\subsection{Multitask Regression}
\label{sec:model-estimator}

For tasks \(t=1,\ldots,T\), observe
\[
        y^{(t)}=X^{(t)}\beta^{*(t)}+e^{(t)},
        \qquad t=1,\ldots,T,
\]
where every task uses the same \(p\) predictors,
\(X^{(t)}\in\mathbb R^{n_t\times p}\), and
\(y^{(t)}\in\mathbb R^{n_t}\).  Write
\[
        B^*=[\beta^{*(1)},\ldots,\beta^{*(T)}]
        \in\mathbb R^{p\times T},
\]
and let \(b_j^*\) denote the vector of coefficients for predictor \(j\).
Intercepts, when included, are treated separately and receive neither penalty.

\begin{assumption}[Gaussian design and noise]
\label{ass:task-specific-design}
The design rows are independent within and across tasks, with
\(x_i^{(t)}\sim N_p(0,\Sigma_t)\).  For fixed constants
\(0<c_0\le C_0<\infty\),
\(c_0I_p\preceq\Sigma_t\preceq C_0I_p\) for every task.
The errors are independent within and across tasks and independent of all
designs, with \(e_i^{(t)}\sim N(0,\sigma_t^2)\), where
\(0<\sigma_t\le\sigma\) and \(\sigma>0\) is a noise upper bound.
\end{assumption}

We allow \(\Sigma_t\) to vary with the task, so the predictor distributions
need not be the same across tasks.  Write \(n_{\min}=\min_t n_t\), and write
\(A\lesssim B\) when \(A\le CB\), where \(C\) depends only on the constants in
Assumption~\ref{ass:task-specific-design}.  We measure estimation error by
\(\|\widehat B-B^*\|_F^2\).

\subsection{Predictor-Specific Sharing}
\label{sec:predictor-specific-sharing}

For predictor \(j\), define
\begin{equation}
\label{eq:predictor-exceptions}
 q_j:=\min_{a\in\mathbb R}
       \|b_j^*-a\mathbf1_T\|_0,
 \qquad j=1,\ldots,p,
\end{equation}
where \(\|\cdot\|_0\) counts nonzero entries.  Thus \(q_j\) is the smallest
number of task coefficients that differ from one shared value.  Equivalently,
\(T-q_j\) is the size of the largest equality group for predictor \(j\).  The
choice is made separately for each predictor, so the exceptional tasks may be
different across predictors.

We summarize the total amount of heterogeneity by
\begin{equation}
\label{eq:total-exceptions}
 r:=\sum_{j=1}^p q_j
   =\min_{a\in\mathbb R^p}
      \|B^*-a\mathbf1_T^\top\|_0.
\end{equation}
We count exceptional coefficients rather than all of the unequal task pairs
they generate.

For the sparse estimator we also use
\[
 S:=\{j:b_j^*\ne0\},\qquad s:=|S|.
\]
Here \(s\) counts active predictors, while \(r\) counts coefficient exceptions
within them.

\subsection{Pairwise Fusion Estimators}
\label{sec:estimators}

We give each task equal weight, rather than weighting tasks by their sample
sizes, and use
\begin{equation}
\label{eq:equal-weight-loss}
 \mathcal L_n(B)
 :=
 \frac{1}{2T}\sum_{t=1}^T
 \frac{\|y^{(t)}-X^{(t)}\beta^{(t)}\|_2^2}{n_t}.
\end{equation}
Because no task ordering or external similarity structure is assumed, we
penalize every pairwise coefficient difference:
\begin{equation}
\label{eq:pairwise-penalty}
 \mathcal J(B)
 :=\sum_{j=1}^p\sum_{1\le t<u\le T}|b_{j,t}-b_{j,u}|.
\end{equation}

Our sparse estimator is
\begin{equation}
\label{eq:mtl_ls_obj}
 \widehat B_{\rm SP}
 \in\arg\min_B
 \left\{
   \mathcal L_n(B)+\lambda\mathcal J(B)+\nu\|B\|_{2,1}
 \right\},
\end{equation}
where \(\|B\|_{2,1}:=\sum_{j=1}^p\|b_j\|_2\).
The group penalty selects predictors, while the pairwise penalty encourages
their coefficients to agree across tasks.  We refer to
\eqref{eq:mtl_ls_obj} as \emph{sparse pairwise fusion}.  Setting
\(\lambda=0\) gives Group Lasso.

When predictor selection is not needed, we use
\begin{equation}
\label{eq:pairwise-fusion-obj}
 \widehat B_{\rm P}
 \in\arg\min_B
 \left\{\mathcal L_n(B)+\lambda\mathcal J(B)\right\}.
\end{equation}
We call this \emph{pairwise fusion}.  This version studies coefficient sharing
without predictor selection and is also used below to study exact pooling of a
shared set of tasks.

\section{MINIMAX-OPTIMAL ESTIMATION}
\label{sec:minimax-estimation}

We next quantify the cost of predictor selection and coefficient heterogeneity.
We assume only that at most \(s\) predictors are active and that there are at
most \(r\) coefficient exceptions in total.  The upper bound separates these
two costs, and the lower bound shows that both are unavoidable.

Fix \(\delta\in(0,\delta_0]\), where \(\delta_0<1\) is numerical.  Constants
in the upper bounds may depend on this fixed confidence level, so
\(\delta\) does not appear in the displayed rates.

\subsection{Upper Bounds for Pairwise Fusion}
\label{sec:pairwise-upper-bounds}

We first study pairwise fusion without predictor selection and then add the
group penalty.

\begin{theorem}[Pairwise fusion]
\label{thm:pairwise-fusion}
Suppose Assumption~\ref{ass:task-specific-design} holds, the penalty includes
all task pairs, and \(1\le r\le p(T-1)\).  Take
\begin{equation}
\label{eq:minimax-fusion-tuning}
        \lambda^2
        =C\frac{\sigma^2}{T^4n_{\min}}
        \log\!\left(\frac{epT}{r}\right).
\end{equation}
Suppose
\begin{equation}
\label{eq:pairwise-sample-size}
 n_{\min}\gtrsim\min\{p,r\}
 \log\!\left(\frac{ep}{\min\{p,r\}}\right)+\log(pT),
 \qquad n_{\min}T\gtrsim p+\log(pT).
\end{equation}
Then, uniformly over all \(B^*\) with at most \(r\) coefficient exceptions,
every minimizer obeys
\begin{equation}
\label{eq:pairwise-dense-rate}
 \|\widehat B_{\rm P}-B^*\|_F^2
 \lesssim
 \frac{\sigma^2}{n_{\min}}
 \left\{p+r\log\!\left(\frac{epT}{r}\right)\right\}
\end{equation}
with probability at least \(1-\delta\).
\end{theorem}

When only \(s\) predictors are active, the group penalty replaces the dense
\(p\)-dimensional cost by a sparse selection cost.

\begin{theorem}[Sparse pairwise fusion]
\label{thm:sparse-pairwise-fusion}
Suppose Assumption~\ref{ass:task-specific-design} holds, the penalty includes
all task pairs, \(1\le s\le p\), and \(1\le r\le s(T-1)\).  Take
\begin{equation}
\label{eq:minimax-group-fusion-tuning}
 \lambda^2=C\frac{\sigma^2}{T^4n_{\min}}
 \log\!\left(\frac{esT}{r}\right),\qquad
 \nu^2=C\frac{\sigma^2}{T^2n_{\min}}
 \left\{\log\!\left(\frac{ep}{s}\right)
 +\frac rs\log\!\left(\frac{esT}{r}\right)\right\}.
\end{equation}
If
\begin{equation}
\label{eq:minimax-group-sample-size}
        n_{\min}\gtrsim s\log\!\left(\frac{ep}{s}\right)+\log(pT),
\end{equation}
then, uniformly over all \(B^*\) with at most \(s\) active predictors and
\(r\) coefficient exceptions, every minimizer obeys
\begin{equation}
\label{eq:pairwise-sparse-rate}
 \|\widehat B_{\rm SP}-B^*\|_F^2
 \lesssim\frac{\sigma^2}{n_{\min}}
 \left\{s\log\!\left(\frac{ep}{s}\right)
 +r\log\!\left(\frac{esT}{r}\right)\right\}
\end{equation}
with probability at least \(1-\delta\).
\end{theorem}

\begin{remark}[Where the rate comes from]
\label{rem:upper-bound-interpretation}
The sparse rate has two parts,
\[
 \frac{\sigma^2}{n_{\min}}s\log(ep/s)
 \quad\text{and}\quad
 \frac{\sigma^2}{n_{\min}}r\log(esT/r).
\]
The first is the familiar cost of finding the active predictors.  For the
second, suppose for intuition that each active predictor has about \(q\)
exceptions, so \(r=sq\).  The term becomes \(sq\log(eT/q)\): there are \(sq\)
departures to estimate and a logarithmic cost for finding their locations.  If
\(q\) is of order \(T\), this becomes order \(sT\), the usual group-sparse
multitask scale \citep{lounici2009taking,lounici2011oracle}.

A single exceptional coefficient creates many unequal task pairs, but the
rate charges it once through \(r\), not once for every pair.  This agrees with
earlier results for pairwise difference penalties
\citep{hutter2016optimal,bellec2017graphslope}.  The theoretical choice of
\(\lambda\) uses \(r\); in practice, our numerical studies choose it by
cross-validation.
\end{remark}

\subsection{Matching Minimax Lower Bound}
\label{sec:minimax-lower-bounds}

To see whether a smaller rate is possible, we consider all estimators under
the same two restrictions on the true coefficient matrix.  For integers
\(1\le s\le p\) and \(1\le r\le s(T-1)\), define
\begin{equation}
\label{eq:predictor-sharing-class}
 \Theta_{\rm share}(s,r):=\!\biggl\{B\in\mathbb R^{p\times T}:
 \left|\{j:b_j\ne0\}\right|\le s,\quad
 \sum_{j=1}^p\min_{a\in\mathbb R}
 \|b_j-a\mathbf1_T\|_0\le r\biggr\}.
\end{equation}
Both the active predictors and the predictor-specific exception locations are
unknown.

\begin{theorem}[Minimax lower bound]
\label{thm:minimax-lower-bounds}
Suppose \(n_t=n\) for all \(t\), the design rows are independent
\(N(0,I_p)\), and the errors are independent \(N(0,\sigma^2)\).
Let \(1\le s\le p\) and \(1\le r\le s(T-1)\).  There is a universal
constant \(c>0\) such that
\begin{equation}
\label{eq:predictor-sharing-lower-bound}
 \inf_{\widehat B}\sup_{B^*\in\Theta_{\rm share}(s,r)}
 \mathbb E_{B^*}\|\widehat B-B^*\|_F^2
 \ge c\frac{\sigma^2}{n}
 \left\{s\log\!\left(\frac{ep}{s}\right)
 +r\log\!\left(\frac{esT}{r}\right)\right\}.
\end{equation}
\end{theorem}

The proof is given in Appendix~\ref{sec:lower-bound-proof}.

\begin{remark}[Interpretation]
\label{rem:minimax-interpretation}
In the balanced isotropic model, the sparse upper and lower bounds have the
same order.  The term \(s\log(ep/s)\) is the usual price of not knowing the
active predictors
\citep{ye2010rate,raskutti2011minimax,bellec2018slope}, while
\(r\log(esT/r)\) is the additional price of not knowing which task
coefficients differ from their predictor-specific shared values.  Thus partial sharing can be substantially easier than unrestricted multitask
regression when \(r\ll sT\), and this advantage shrinks as \(r\) approaches
order \(sT\).

The upper bound is stated with high probability, while the lower bound is for
expected risk.  Matching here refers to their dependence on
\((n,p,T,s,r)\) at a fixed confidence level.
\end{remark}

\section{POOLING A SHARED SET OF TASKS}
\label{sec:shared-task-taskwise}

The preceding theory allows the exceptional tasks to differ from predictor to
predictor.  It therefore does not require any two full task coefficient
vectors to be identical.  We now consider the stronger situation in which a
large set of tasks truly shares one coefficient vector and ask whether the
same estimator finds that equality automatically.

\subsection{A Shared Set of Tasks}
\label{sec:shared-task-special-case}

Let \(\mathcal C\subseteq\{1,\ldots,T\}\) contain \(T-q\) tasks, where
\(0\le q\le T/2\), and suppose
\[
 \beta^{*(t)}=
 \begin{cases}
 \beta_0^*,&t\in\mathcal C,\\
 \beta_0^*+\delta^{*(t)},&t\in\mathcal C^c.
 \end{cases}
\]
Thus the tasks in \(\mathcal C\) share the same full coefficient vector
\(\beta_0^*\), while the remaining \(q\) tasks may differ.  The shared vector
can be dense.  Let
\[
 s_{\Delta}=\left|\bigcup_{u\in\mathcal C^c}
                    \operatorname{supp}(\delta^{*(u)})\right|
\]
be the number of predictors on which at least one outside task differs.
The deviations may involve different predictors and have different
magnitudes, and we impose no separation from the shared tasks.  For \(q=0\), set \(s_{\Delta}=0\).  This setting
implies \(r\le q s_{\Delta}\) in the preceding predictor-specific formulation.

\subsection{Exact Pooling of the Shared Tasks}
\label{sec:shared-task-pooling}

If \(\mathcal C\) were known, one could pool its observations to estimate
\(\beta_0^*\) while fitting the remaining tasks separately.  The next theorem
shows that pairwise fusion can recover this behavior without being given
\(\mathcal C\).

\begin{theorem}[Exact pooling of shared tasks]
\label{thm:shared-task-pooling}
Under Assumption~\ref{ass:task-specific-design} and the shared-task setting
above, assume
\begin{equation}
\label{eq:shared-task-pooling-conditions}
 T\gtrsim q\sqrt{s_{\Delta}},\qquad
 n_{\min}\gtrsim(1+s_{\Delta}q)\log(epT),\qquad
 n_{\min}T\gtrsim p\bigl(1+q\sqrt{s_{\Delta}}\bigr).
\end{equation}
Choose
\[
 \lambda^2=
 C_\lambda\frac{\sigma^2}{T^4n_{\min}}
 \log(epT)
\]
for a sufficiently large fixed numerical constant \(C_\lambda\).  Then,
with probability at least \(1-(pT)^{-2}\), every pairwise-fusion minimizer has
\(\widehat\beta_{\rm P}^{(t)}=\widehat\beta_{\rm P}^{(u)}\) for all
\(t,u\in \mathcal{C}\) and satisfies
\begin{equation}
\label{eq:shared-task-error}
 \max_{t\in\mathcal C}
 \|\widehat\beta_{\rm P}^{(t)}-\beta^{*(t)}\|_2^2
 \lesssim\frac{\sigma^2}{n_{\min}T}
 \left[p+\log T+\frac{s_{\Delta}q^2}{T}\log(epT)\right].
\end{equation}
Moreover, if \(q\ge1\), the remaining tasks satisfy
\begin{equation}
\label{eq:outside-task-average-error}
 \frac1q\sum_{u\in\mathcal C^c}
 \|\widehat\beta_{\rm P}^{(u)}-\beta^{*(u)}\|_2^2
 \lesssim\frac{\sigma^2}{n_{\min}}
 \left\{s_{\Delta}\log(epT)+\frac{p+\log T}{T}\right\}.
\end{equation}
\end{theorem}

The proof is given in Appendix~\ref{sec:taskwise-proof}.

\begin{remark}[What pooling buys]
\label{rem:taskwise-interpretation}
Every minimizer gives the same coefficient vector to all tasks in
\(\mathcal C\).  If \(\mathcal C\) were known, pooling its observations would
give an error of order
\[
        \frac{\sigma^2(p+\log T)}{N_{\mathcal C}},
        \qquad N_{\mathcal C}=\sum_{u\in\mathcal C}n_u.
\]
When task sample sizes are comparable, this matches the first term in
\eqref{eq:shared-task-error}.  The extra term comes from estimating the
remaining tasks at the same time, and it is no larger than the pooled term
whenever
\begin{equation}
\label{eq:taskwise-oracle-comparison}
 \frac{s_\Delta q^2}{T}\log(epT)\lesssim p+\log T.
\end{equation}

This result is useful when a single task cannot estimate a dense
\(p\)-dimensional vector well.  If \(q\) and \(s_{\Delta}\) are bounded, the
sufficient conditions reduce, up to logarithms, to
\[
        n_{\min}\gtrsim\log(epT),\qquad n_{\min}T\gtrsim p.
\]
Thus the shared vector can be estimated from pooled information even when
\(n_{\min}\ll p\).

For a task outside \(\mathcal C\), the error combines the shared-vector error
with the error in its sparse departure, which gives
\eqref{eq:outside-task-average-error}.  The same proof also gives
\[
 \max_{u\in\mathcal C^c}
 \|\widehat\beta_{\rm P}^{(u)}-\beta^{*(u)}\|_2^2
 \lesssim\frac{\sigma^2}{n_{\min}}
 \left\{s_{\Delta}q\log(epT)+\frac{p+\log T}{T}\right\}.
\]
The theorem guarantees equality within \(\mathcal C\); it does not require
the outside tasks to remain separated from \(\mathcal C\), which would need an
additional separation condition.
\end{remark}

\section{COMPUTATION}
\label{sec:computation}

The penalty contains \(T(T-1)/2\) pairwise differences per predictor, but
these differences do not need to be formed explicitly.  We use the
sorting representation and isotonic-regression proximal map of
\citet{lin2019clustered}.  Their clustered-Lasso problem uses the same pairwise
penalty for one coefficient vector.  We apply its proximal step separately to
each predictor, followed by group shrinkage when \(\nu>0\).

To compute our estimator, write \eqref{eq:mtl_ls_obj} as \(f(B)+g(B)\).  We
use FISTA with adaptive restart and step size \(\eta=0.98/L\), where \(L\) is
the largest eigenvalue of the Hessian of \(f\):
\[
 Z^k=Y^k-\eta\nabla f(Y^k),\qquad
 B^{k+1}=\operatorname{prox}_{\eta g}(Z^k).
\]
For the coefficient vector \(z_j\) of predictor \(j\), the fusion-only proximal map of
\citet{lin2019clustered} sorts
\(z_{j,(1)}\leq\cdots\leq z_{j,(T)}\), applies isotonic regression to
\(z_{j,(i)}-\eta\lambda(2i-T-1)\), and restores the original order; denote the
result by \(\widetilde b_j\).  The proximal map needed here is
\[
 \operatorname{prox}_{\eta g_j}(z_j)
 =\left(1-\frac{\eta\nu}{\|\widetilde b_j\|_2}\right)_+
   \widetilde b_j.
\]
The second step is ordinary group shrinkage.  The composition is
exact because the fusion penalty is convex and positively homogeneous, so its
subdifferential is unchanged when a nonzero vector is rescaled along a
positive ray.  This identity is used only for computation; the algorithmic
contribution is not a focus of the paper.  Each predictor update costs
\(O(T\log T)\), dominated by sorting, rather than explicitly manipulating all
\(T(T-1)/2\) pairwise differences.  Group Lasso uses the same iteration with
\(\lambda=0\), while the Lasso baselines use coordinate descent.  We declare a fused fit converged only when both the relative iterate change
and a normalized proximal fixed-point residual fall below the prescribed
tolerance.

\section{SIMULATIONS}
\label{sec:simulation}
\label{subsec:simulation-structure}

We vary the fraction of task coefficients that depart from a shared value
and measure the resulting estimation error.  We select \(s\) active predictors.
For each active predictor, start with one random-sign coefficient
shared by all \(T\) tasks, then select a fraction \(\alpha\) of its task
positions uniformly at random and add independent \(N(0,4)\) deviations.
The selected positions are independent across predictors; inactive predictors stay
zero.  At \(\alpha=0\), all tasks share one coefficient vector.  At
\(\alpha=1\), each active predictor has distinct coefficients almost surely.
Thus the experiment moves from full sharing to no exact coefficient sharing,
without imposing one shared set of tasks.

Table~\ref{tab:sharing-results} gives the two dimension regimes and the seven
departure fractions.  Designs are Gaussian with
\(\Sigma_{jk}=0.5^{|j-k|}\), and the noise standard deviation is \(1.2\).
We normalize the coefficient vectors of the active predictors to equal energy
and then rescale the full matrix
to keep the average population signal-to-noise ratio at \(4\).  This preserves
coefficient equalities, with exception count
\(r=s\min\{\alpha T,T-1\}\).  Each repetition uses nested departure
locations and the same designs and noise across fractions.

The comparison methods use separate fits, complete pooling, or a shared active
support.  The fused methods learn coefficient equalities from the data.
Within each setting, all penalized methods are tuned by matched five-fold
prediction cross-validation on an independent pilot dataset, then evaluated
on 30 new datasets with those penalties fixed.  The table reports only
\(\|\widehat B-B^*\|_F^2\), summed over all predictors and tasks, as a mean
and Monte Carlo standard error.  Appendix~\ref{subsec:random-departure-setup}
gives the generation and tuning details.

\begin{table}[!t]
\centering
\small
\setlength{\tabcolsep}{8pt}
\caption{Squared Frobenius error \(\|\widehat B-B^*\|\sb F^2\) as the fraction of randomly perturbed coefficients on the active support increases. Entries are Monte Carlo means (standard errors) over 30 repetitions. The error is summed over all predictors and tasks.}
\label{tab:sharing-results}
\textbf{(a) Low dimension}: \((p,s,n,T)=(40,40,140,60)\)\par\smallskip
\begin{tabular}{@{}rcccc@{}}
\toprule
Departures & \shortstack{Separate\\OLS} & \shortstack{Separate\\Ridge} & \shortstack{Pooled\\Ridge} & \shortstack{Pairwise\\Fusion} \\
\midrule
0\% & \(57.66\;(0.46)\) & \(46.43\;(0.45)\) & \(0.68\;(0.04)\) & \(0.68\;(0.04)\) \\
10\% & \(57.66\;(0.46)\) & \(46.35\;(0.40)\) & \(95.02\;(2.27)\) & \(12.33\;(0.14)\) \\
20\% & \(57.66\;(0.46)\) & \(46.08\;(0.58)\) & \(150.40\;(2.74)\) & \(24.33\;(0.30)\) \\
40\% & \(57.66\;(0.46)\) & \(46.32\;(0.49)\) & \(211.38\;(2.52)\) & \(31.87\;(0.26)\) \\
60\% & \(57.66\;(0.46)\) & \(46.28\;(0.43)\) & \(243.45\;(2.38)\) & \(38.55\;(0.29)\) \\
80\% & \(57.66\;(0.46)\) & \(46.35\;(0.42)\) & \(262.86\;(2.12)\) & \(42.91\;(0.33)\) \\
100\% & \(57.66\;(0.46)\) & \(46.54\;(0.42)\) & \(277.27\;(1.80)\) & \(45.45\;(0.37)\) \\
\bottomrule
\end{tabular}
\par\medskip
\textbf{(b) High dimension}: \((p,s,n,T)=(180,12,80,60)\)\par\smallskip
\begin{tabular}{@{}rccccc@{}}
\toprule
Departures & \shortstack{Separate\\Lasso} & \shortstack{Pooled\\Ridge} & \shortstack{Pooled\\Lasso} & \shortstack{Group\\Lasso} & \shortstack{Sparse Pairwise\\Fusion} \\
\midrule
0\% & \(128.71\;(2.04)\) & \(5.71\;(0.10)\) & \(1.50\;(0.07)\) & \(33.25\;(0.57)\) & \(0.93\;(0.05)\) \\
10\% & \(116.94\;(1.50)\) & \(104.48\;(3.23)\) & \(95.66\;(2.98)\) & \(33.34\;(0.45)\) & \(8.34\;(0.15)\) \\
20\% & \(107.92\;(1.04)\) & \(160.60\;(3.26)\) & \(149.52\;(3.05)\) & \(33.23\;(0.41)\) & \(14.28\;(0.19)\) \\
40\% & \(94.09\;(0.76)\) & \(223.88\;(2.27)\) & \(211.24\;(2.29)\) & \(32.78\;(0.34)\) & \(23.72\;(0.26)\) \\
60\% & \(86.97\;(0.67)\) & \(256.13\;(1.92)\) & \(243.19\;(1.92)\) & \(33.04\;(0.39)\) & \(26.22\;(0.38)\) \\
80\% & \(83.39\;(0.67)\) & \(276.97\;(1.59)\) & \(264.01\;(1.61)\) & \(33.02\;(0.39)\) & \(28.26\;(0.42)\) \\
100\% & \(82.60\;(0.73)\) & \(289.41\;(1.44)\) & \(277.25\;(1.46)\) & \(32.85\;(0.37)\) & \(29.08\;(0.39)\) \\
\bottomrule
\end{tabular}
\end{table}

In both regimes, the fused error increases as departures become more common.
In low dimension, Pairwise Fusion rises from \(0.68\) at \(0\%\) departures
to \(12.33\) at \(10\%\) and \(45.45\) at \(100\%\), while Separate
Ridge stays near \(46\).  In high dimension, Sparse Pairwise Fusion rises from
\(0.93\) to \(8.34\) and \(29.08\) at the same fractions, while Group
Lasso stays near \(33\).  The gain from fusion is therefore largest when most coefficients retain their
shared value and shrinks as exact sharing becomes less common.

Complete pooling is competitive at \(0\%\) departures but deteriorates
rapidly once task coefficients differ.  At \(100\%\), Pooled Ridge has
error \(277.27\) in low dimension and Pooled Lasso has error \(277.25\)
in high dimension.  Pairwise fusion avoids forcing these differing coefficients to be equal.  At
\(100\%\), where there are no true equalities to recover, any remaining gain
comes from finite-sample shrinkage rather than from recovering shared
coefficients.

\section{REAL DATA: HOUSEHOLD ELECTRICITY USE}
\label{sec:real-data-recs}

We analyze the 2020 Residential Energy Consumption Survey (RECS)
\citep{eia2020recs}, retaining all 18,496 households in the public-use file.
The 50 states and the District of Columbia form \(T=51\) tasks, with
143--1,152 households per task.  The response is log annual electricity use
in kilowatt-hours.  The 16 predictors describe electric heating, electric
water heating, air conditioning, floor area, household size, climate,
housing type, income, and construction period.  Categorical variables use
fixed dummy encodings; Appendix~\ref{sec:recs-details} gives their definitions.

We fit the survey-weighted version of Pairwise Fusion in
\eqref{eq:pairwise-fusion-obj}, with \(\nu=0\) and an unpenalized intercept for
each state.  Survey weights are normalized within each state, and states
receive equal weight in the loss.  Ordinary five-fold cross-validation splits
households within each state: four folds train the model and the fifth
validates it.  We select the fusion parameter with the smallest mean
validation error, then refit on all households to obtain one final
\(\widehat B\).  No group penalty or predictor selection is used.

The selected fit contains many exact coefficient ties.  All 16 predictors vary
across states, but 14 have a majority equality group.  The median largest group contains 40
of the 51 tasks, and seven predictors share a coefficient in at least 80\% of
tasks.  For each predictor, let \(\widehat a_j\) be the mean coefficient in its
largest numerical equality group, on the original response scale.  We display
\(\widehat\delta_{jt}=\widehat\beta^{\rm raw}_{jt}-\widehat a_j\),
scaled by the common predictor standard deviation \(s_j\).
This is only a summary of the final fit; \(\widehat a_j\) is not a separately
estimated common component.

\begin{figure}[!t]
\centering
\includegraphics[width=\textwidth]{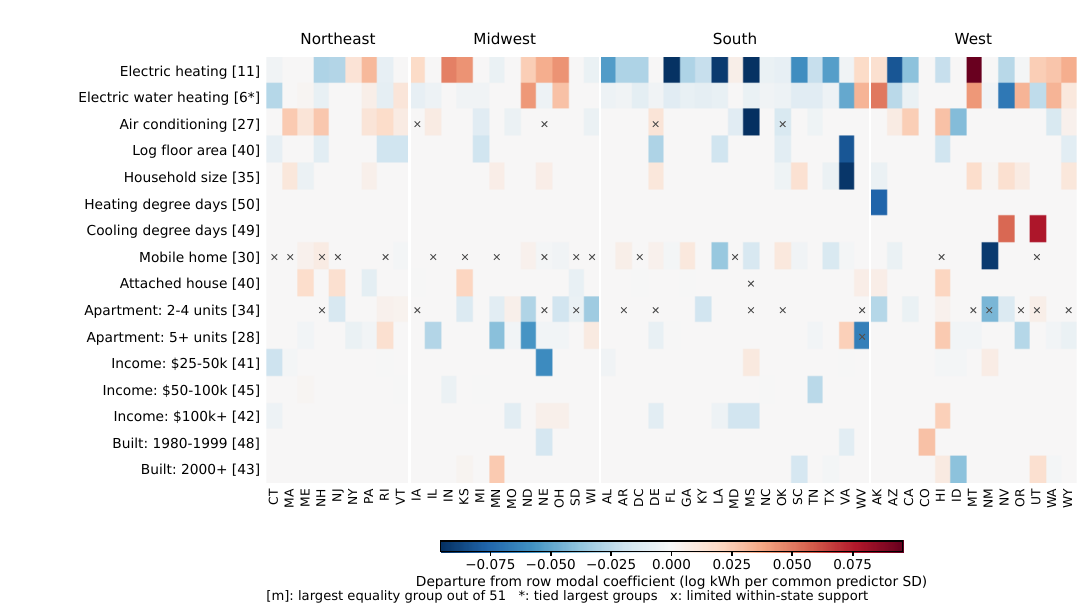}
\caption{RECS fitted coefficient departures, retaining all 16 predictors and
all 50 states plus DC.  Columns are alphabetical within Census regions;
regions organize the display only.  Color is \(s_j\widehat\delta_{jt}\),
in log-kWh units per common predictor standard deviation, with no masking
or clipping.  Brackets give the largest equality-group size out of 51;
an asterisk denotes tied largest groups.  Crosses mark fewer than ten
households in one binary category or no within-state predictor variation.
These annotations do not remove cells.  Appendix~\ref{sec:recs-details}
defines the baseline and numerical equality rule.}
\label{fig:recs-results}
\end{figure}

Figure~\ref{fig:recs-results} shows that the sharing sets differ by predictor.
Heating degree days have one shared coefficient in 50 tasks, with Alaska
departing; cooling degree days share a coefficient in 49 tasks, with Nevada
and Utah departing.  Floor area has a 40-task equality group.  These predictors
give concrete examples of broad sharing with different local exceptions.
The exceptions remain conditional on the other predictors, including the
equipment indicators.

Equipment coefficients are more dispersed.  The electric-heating coefficient
ranges from 0.139 in Florida and Mississippi to 0.547 in Montana, while its
largest equality group has only 11 states.  Electric water heating has 25
coefficient groups, with two largest groups of six states each.  Air
conditioning instead has a 27-task majority group.  Thus the same fit shows both broad sharing and substantial state-to-state
variation; we did not select predictors to highlight one pattern.

These are conditional associations, not causal equipment effects or evidence
of exact population equality.  Some states have limited within-state
contrasts, and the survey's weighting and imputation differ from the
independent Gaussian design in our theory.  The analysis illustrates the
structure of a fitted coefficient matrix without claiming that every
observed departure is a stable population feature.

\section{DISCUSSION}
\label{sec:discussion}

This paper studies multitask regression when coefficient sharing can differ by
predictor.  Two quantities determine the main rate: \(s\), the number of
active predictors, and \(r\), the number of task coefficients that depart from
their predictor-specific shared values.  The sparse upper and lower bounds
have the common scale
\[
        \frac{\sigma^2}{n}
        \{s\log(ep/s)+r\log(esT/r)\}.
\]
Thus the benefit of sharing is largest when only a small fraction of
coefficients differ across tasks, and it fades continuously as the
coefficients become more task specific.

A second result considers a stronger setting in which a large set of tasks
shares the entire coefficient vector.  Pairwise fusion can pool those tasks
exactly without being told which tasks they are.  Its error for the shared
vector is a pooled estimation term plus a controlled contribution from the
remaining tasks.

Our theory has several limitations.  The theoretical tuning levels use \(s\) and
\(r\), while the numerical work uses cross-validation.  Exact pooling assumes exact equality among the shared tasks and does not yet cover
approximate sharing.  The Gaussian independent-task design also excludes
heavy tails, dependence across tasks, and settings with useful task ordering
or external similarity information.

The RECS analysis illustrates the predictor-specific view directly.  Climate
and several housing variables have large equality groups across states, while
some equipment coefficients are more dispersed.  We retain all predictors
and states so that both patterns remain visible.  Because the data use survey
weights and some states have limited local contrasts, the fitted equality
groups are descriptive rather than claims of exact population equality.

\section*{AI Use Statement}

During the preparation of this work, the authors used OpenAI Codex to assist
with improving mathematical proofs, implementing numerical experiments, and
improving the exposition and organization of the manuscript.  The authors
reviewed, edited and verified all AI-assisted material and take full
responsibility for the content of the publication.

\appendix

\section{PROOFS OF THE UPPER BOUNDS}
\label{sec:main-proofs}

Both upper bounds use the same basic argument.  Optimality controls the loss,
the penalties restrict the estimation error, and the design converts this
control into a bound on the coefficients.  Pairwise Fusion requires one
additional step because the common coefficient across tasks is not penalized.
We first state the two probabilistic ingredients and prove the main results;
the supporting bounds appear in Section~\ref{subsec:upper-bound-details}.
The minimax lower bound and the shared-task result are proved separately in
Appendices~\ref{sec:lower-bound-proof} and~\ref{sec:taskwise-proof}.

Constants may change from line to line.  The confidence level \(\delta\) is
fixed throughout.  When several events are combined, we assign each a fixed
fraction of \(\delta\) and enlarge constants as needed; the resulting event
has probability at least \(1-\delta\).  Because the bounds are uniform over
the error direction, they apply simultaneously to every minimizer.

\subsection{Bounds Used in the Upper-Bound Proofs}
\label{sec:proof-notation}
\label{sec:common-proof-tools}

For each task, write
\[
 \widehat\Sigma_t=\frac{X^{(t)\top}X^{(t)}}{n_t},\qquad
 w_t=\frac{X^{(t)\top}e^{(t)}}{n_t},\qquad
 Z=\frac1T[w_1,\ldots,w_T].
\]
Then \(\langle Z,H\rangle=T^{-1}\sum_t w_t^\top h^{(t)}\) is the noise
term in the loss expansion.  Let \(D\) denote the pairwise-difference matrix,
with one row \(e_t^\top-e_u^\top\) for each \(t<u\).
For a true coefficient matrix, put
\[
 S=\{j:b_j^*\ne0\},\qquad
 \mathcal A=\operatorname{supp}(B^*D^\top).
\]
Thus \(S\) indexes the active predictors, while \(\mathcal A\) indexes the
unequal task pairs for those predictors.

\paragraph{Two facts about pairwise differences.}
The incidence matrix satisfies
\begin{equation}
\label{eq:complete-incidence-identity}
 D^\top D=TI_T-\mathbf1_T\mathbf1_T^\top,\qquad
 \|HD^\top\|_F^2=T\|H\|_F^2-\|H\mathbf1_T\|_2^2.
\end{equation}
For predictor \(j\), choose one of its most frequent coefficient values.
Every unequal pair contains at least one of the \(q_j\) exceptions, so there
are at most \(Tq_j\) such pairs.  Consequently,
\begin{equation}
\label{eq:full-cut-compatibility-bound}
 |\mathcal A|\le Tr,\qquad
 \|(HD^\top)_{\mathcal A}\|_1
 \le\sqrt{|\mathcal A|}\,\|HD^\top\|_F
 \le T\sqrt r\,\|H\|_F.
\end{equation}
For completeness, if \(n_{j1},\ldots,n_{jK_j}\) are the multiplicities
of the distinct coefficients for predictor \(j\), then
\(\max_k n_{jk}=T-q_j\) and
\[
 |A_j|=\frac12\left(T^2-\sum_k n_{jk}^2\right)
 \ge\frac12\{T^2-T(T-q_j)\}=\frac{Tq_j}{2}.
\]
This gives the two pair-count bounds used in the main text.

The second fact controls a centered coefficient vector by its pairwise
differences.  If \(\sum_t g_t=0\), then \(Tg_t=\sum_u(g_t-g_u)\).
Taking absolute values and summing first over tasks and then over predictors
gives
\begin{equation}
\label{eq:pairwise-centering-l1}
 G\mathbf1_T=0
 \quad\Longrightarrow\quad
 T\|G\|_1\le2\|GD^\top\|_1.
\end{equation}
Both statements are deterministic.

\paragraph{The inequality supplied by optimality.}
Let \(H=\widehat B-B^*\).  Comparing the objective at \(\widehat B\)
and \(B^*\), then expanding the squared loss, gives
\begin{equation}
\label{eq:global-basic-ineq}
\begin{aligned}
 \predictionerror
 &+\nu\{\|B^*+H\|_{2,1}-\|B^*\|_{2,1}\}\\
 &+\lambda\{\|(B^*+H)D^\top\|_1-\|B^*D^\top\|_1\}
 \le\langle Z,H\rangle.
\end{aligned}
\end{equation}
For Pairwise Fusion, set \(\nu=0\).  The next two lemmas provide the
design and noise bounds used with this inequality.

\begin{lemma}[Two design bounds]
\label{lem:upper-design-bounds}
\label{lem:row-cone-curvature}
\label{lem:fusion-pooled-curvature}
Fix \(K\ge1\), \(k\ge1\), and let \(\bar k=\min\{p,k\}\).  Under
Assumption~\ref{ass:task-specific-design}, suppose
\begin{equation}
\label{eq:row-cone-curvature-sample}
 n_{\min}\ge C_K\{\bar k\log(ep/\bar k)+\log(T/\delta)\}.
\end{equation}
Then, with probability at least \(1-\delta\), the following bounds hold.

\emph{(i) Errors concentrated on a few predictors.}
Every \(H\) satisfying
\begin{equation}
\label{eq:row-cone-curvature-cone}
 \|H\|_{2,1}\le K\sqrt{\bar k}\,\|H\|_F
\end{equation}
obeys
\begin{equation}
\label{eq:row-cone-curvature-lower}
 \frac cT\|H\|_F^2\le\predictionerror\le\frac CT\|H\|_F^2.
\end{equation}

\emph{(ii) A common vector and a few varying predictors.}
If, in addition,
\begin{equation}
\label{eq:fusion-pooled-curvature-sample}
 n_{\min}T\ge C_K\{p+\log(1/\delta)\},
\end{equation}
then every \(a\in\mathbb R^p\) and centered \(G\) satisfying
\(\|G\|_{2,1}\le K\sqrt{\bar k}\,\|G\|_F\) obey
\begin{equation}
\label{eq:fusion-pooled-curvature}
 c\{T\|a\|_2^2+\|G\|_F^2\}
 \le\sum_t\frac{\|X^{(t)}(a+g^{(t)})\|_2^2}{n_t}
 \le C\{T\|a\|_2^2+\|G\|_F^2\}.
\end{equation}
Here centered means \(G\mathbf1_T=0\).  Constants may depend on \(K\).
\end{lemma}

Part (i) is the standard predictor-sparse design bound used for Group Lasso
\citep{lounici2009taking,lounici2011oracle}.  Part (ii) also allows the common
vector to be dense: task-specific samples control the varying predictors,
while the pooled sample controls the shared component.

For the next lemma, write
\[
 \widehat\Sigma_\bullet=\frac1T\sum_t\widehat\Sigma_t,
 \qquad \bar w=\frac1T\sum_t w_t.
\]
When \(\widehat\Sigma_\bullet\) is invertible, define
\begin{equation}
\label{eq:common-noise-and-residual}
 h_N=\widehat\Sigma_\bullet^{-1}\bar w,\qquad
 \zeta_t=\frac1T(w_t-\widehat\Sigma_t h_N),\qquad
 \zeta=[\zeta_1,\ldots,\zeta_T].
\end{equation}
The vector \(h_N\) is the error obtained by fitting one common coefficient
to noise alone.  The matrix \(\zeta\) is the residual noise after this fit,
and \(\zeta\mathbf1_T=0\).

\begin{lemma}[Noise bounds for the two estimators]
\label{lem:upper-noise-bounds}
\label{lem:pairwise-score}
\label{lem:localized-group-fusion-score}
Under Assumption~\ref{ass:task-specific-design}, each of the following
statements holds with probability at least \(1-\delta\).

\nopagebreak[4]
\emph{(i) Pairwise Fusion.}
Under the sample-size conditions and tuning of
Theorem~\ref{thm:pairwise-fusion}, the pooled covariance is invertible,
\begin{equation}
\label{eq:fusion-pooled-noise-rate}
 T\|h_N\|_2^2\lesssim\frac{\sigma^2p}{n_{\min}},
\end{equation}
and, simultaneously for all centered \(G\),
\begin{equation}
\label{eq:fusion-profiled-score}
 |\langle\zeta,G\rangle|
 \le\frac\lambda8\|GD^\top\|_1
       +C\lambda T\sqrt r\,\|G\|_F.
\end{equation}

\emph{(ii) Sparse Pairwise Fusion.}
Under the sample-size condition and tuning of
Theorem~\ref{thm:sparse-pairwise-fusion}, simultaneously for every \(H\),
\begin{equation}
\label{eq:localized-group-fusion-score}
 |\langle Z,H\rangle|
 \le\frac\nu8\|H\|_{2,1}
       +\frac\lambda8\|HD^\top\|_1
       +C\nu\sqrt s\,\|H\|_F.
\end{equation}
This part does not require an invertible pooled covariance.
\end{lemma}

These bounds split the noise into a part absorbed by the penalties and a
remainder controlled in Frobenius norm.  The split is made in coefficient
coordinates rather than over all pairwise differences.  Section~\ref{subsec:upper-bound-details}
controls the remainder at the stated tuning levels; a simple maximum-coordinate
bound would introduce larger logarithmic factors.

\subsection{Proof of Theorem~\ref{thm:pairwise-fusion}}

\begin{proof}
Work on the joint event from the two preceding lemmas.  We first control the
differences across tasks and then recover the error in the shared component.

\paragraph{Step 1: Fit the common component.}
Write
\[
 H=h\mathbf1_T^\top+G,\qquad G\mathbf1_T=0.
\]
The pairwise penalty depends only on \(G\).  For fixed \(G\), minimizing
the loss over \(h\) gives
\[
 h=h_N+h_G,\qquad
 h_G=-\widehat\Sigma_\bullet^{-1}
          \frac1T\sum_t\widehat\Sigma_tg^{(t)}.
\]
Define the remaining quadratic loss by
\begin{equation}
\label{eq:dense-profiled-quadratic-definition}
 Q_X(G)=\min_{a\in\mathbb R^p}\frac1{2T}
          \sum_t(a+g^{(t)})^\top\widehat\Sigma_t(a+g^{(t)}).
\end{equation}
The minimum is attained at \(a=h_G\).  After completing the square, the
part of the profiled loss that depends on \(G\) is
\(Q_X(G)-\langle\zeta,G\rangle\), up to an additive constant.  Compare
the fitted \(G\) with \(G=0\), allowing the common component to be refitted
in each case.  Since
\[
 \|(B^*+G)D^\top\|_1-\|B^*D^\top\|_1
 \ge\|(GD^\top)_{\mathcal A^c}\|_1
       -\|(GD^\top)_{\mathcal A}\|_1,
\]
the noise bound \eqref{eq:fusion-profiled-score} gives
\begin{equation}
\label{eq:minimax-fusion-profiled-basic}
 Q_X(G)+c\lambda\|(GD^\top)_{\mathcal A^c}\|_1
 \le C\lambda\|(GD^\top)_{\mathcal A}\|_1
       +C\lambda T\sqrt r\,\|G\|_F.
\end{equation}

\paragraph{Step 2: Bound the centered differences.}
By \eqref{eq:full-cut-compatibility-bound}, the contribution from unequal
true pairs is at most \(T\sqrt r\,\|G\|_F\).  Dropping
\(Q_X(G)\ge0\) from \eqref{eq:minimax-fusion-profiled-basic} and then adding
the unequal-pair contribution gives
\[
 \|GD^\top\|_1\lesssim T\sqrt r\,\|G\|_F.
\]
The centering inequality \eqref{eq:pairwise-centering-l1} therefore
implies \(\|G\|_{2,1}\le\|G\|_1\lesssim\sqrt r\,\|G\|_F\).
Also \(\|G\|_{2,1}\le\sqrt p\,\|G\|_F\), so
\[
 \|G\|_{2,1}\lesssim\sqrt{\min\{p,r\}}\,\|G\|_F.
\]
Part (ii) of Lemma~\ref{lem:upper-design-bounds}, with
\(k=\min\{p,r\}\), now gives
\(Q_X(G)\gtrsim T^{-1}\|G\|_F^2\).  Substitution into
\eqref{eq:minimax-fusion-profiled-basic} yields
\[
 \frac1T\|G\|_F^2\lesssim\lambda T\sqrt r\,\|G\|_F.
\]
Consequently,
\begin{equation}
\label{eq:minimax-centered-rate}
 \|G\|_F^2\lesssim\lambda^2T^4r
 \asymp\frac{\sigma^2}{n_{\min}}r\log(epT/r).
\end{equation}
If \(G=0\), the bound is trivial; otherwise divide by \(\|G\|_F\).

\paragraph{Step 3: Add back the common-component error.}
The same design bound, together with the minimizing property of \(h_G\),
gives
\[
 T\|h_G\|_2^2+\|G\|_F^2
 \lesssim 2TQ_X(G)
 \le\sum_t g^{(t)\top}\widehat\Sigma_tg^{(t)}
 \lesssim\|G\|_F^2.
\]
The common and centered components are orthogonal.  Thus
\[
 \|H\|_F^2=T\|h_N+h_G\|_2^2+\|G\|_F^2
 \lesssim T\|h_N\|_2^2+\|G\|_F^2.
\]
Combining \eqref{eq:fusion-pooled-noise-rate} and
\eqref{eq:minimax-centered-rate} yields \eqref{eq:pairwise-dense-rate}.
Because the argument is uniform on the joint event, the bound holds for every
minimizer.
\end{proof}

\subsection{Proof of Theorem~\ref{thm:sparse-pairwise-fusion}}

\begin{proof}
Let \(H=\widehat B_{\rm SP}-B^*\).  Here the group penalty already
controls the shared task direction, so the separate profiling step used above
is unnecessary.

\paragraph{Step 1: Use the two penalties.}
The triangle inequality gives
\[
\begin{aligned}
 \|B^*+H\|_{2,1}-\|B^*\|_{2,1}
 &\ge\|H_{S^c}\|_{2,1}-\|H_S\|_{2,1},\\
 \|(B^*+H)D^\top\|_1-\|B^*D^\top\|_1
 &\ge\|(HD^\top)_{\mathcal A^c}\|_1
          -\|(HD^\top)_{\mathcal A}\|_1.
\end{aligned}
\]
Substitute these inequalities and \eqref{eq:localized-group-fusion-score}
into \eqref{eq:global-basic-ineq}, and use
\(\|H_S\|_{2,1}\le\sqrt s\,\|H\|_F\),
\eqref{eq:full-cut-compatibility-bound}, and
\(\lambda T\sqrt r\lesssim\nu\sqrt s\), which follows from the stated
tuning.  Absorbing the smaller penalty terms gives
\begin{equation}
\label{eq:minimax-group-basic-simplified}
 \predictionerror+c\nu\|H_{S^c}\|_{2,1}
 +c\lambda\|(HD^\top)_{\mathcal A^c}\|_1
 \le C\nu\sqrt s\,\|H\|_F.
\end{equation}

\paragraph{Step 2: Convert the loss bound into a coefficient bound.}
Dropping the prediction and pairwise-difference terms gives
\(\|H_{S^c}\|_{2,1}\lesssim\sqrt s\,\|H\|_F\).  Including the active
predictors then yields
\[
 \|H\|_{2,1}\lesssim\sqrt s\,\|H\|_F.
\]
Part (i) of Lemma~\ref{lem:upper-design-bounds}, with \(k=s\), therefore
bounds the prediction term below by \(c\|H\|_F^2/T\).  Returning to
\eqref{eq:minimax-group-basic-simplified} gives
\[
 \frac1T\|H\|_F^2\lesssim\nu\sqrt s\,\|H\|_F,
 \qquad
 \|H\|_F^2\lesssim T^2\nu^2s.
\]
Substituting the tuning level in \eqref{eq:minimax-group-fusion-tuning} gives
\[
 \|H\|_F^2\lesssim\frac{\sigma^2}{n_{\min}}
 \{s\log(ep/s)+r\log(esT/r)\}.
\]
The design and noise events hold jointly with probability at least
\(1-\delta\), and the argument is uniform over all minimizers.
\end{proof}

\subsection{Proofs of the Design and Noise Bounds}
\label{subsec:upper-bound-details}

This subsection proves the two lemmas used above.  All probability bounds are
uniform over the relevant supports, so supports selected from the observed
error or noise do not require an additional union bound.

\paragraph{Probability bounds used below.}
For independent centered sub-exponential variables with
\(\|U_i\|_{\psi_1}\le K\), weighted Bernstein's inequality gives
\begin{equation}
\label{eq:weighted-bernstein}
 \Pr\!\left\{\left|\sum_i a_iU_i\right|>x\right\}
 \le2\exp\!\left[-c\min\!\left\{
 \frac{x^2}{K^2\|a\|_2^2},\frac{x}{K\|a\|_\infty}\right\}\right].
\end{equation}
A centered Gaussian variable has sub-Gaussian norm bounded by a constant
times its standard deviation.  We use these inequalities conditionally on the
design when needed.  An \(\eta\)-net of the unit sphere in
\(\mathbb R^d\) has at most \((1+2/\eta)^d\) points; for such a net
\(\mathcal N\),
\begin{equation}
\label{eq:finite-net-comparisons}
 \|z\|_2\le\frac{\max_{v\in\mathcal N}|v^\top z|}{1-\eta},\qquad
 \|A\|_{\rm op}\le
 \frac{\max_{v\in\mathcal N}|v^\top Av|}{1-2\eta}
 \quad(\eta<1/2)
\end{equation}
for symmetric \(A\).

\begin{proof}[Proof of Lemma~\ref{lem:upper-design-bounds}]
Write
\[
 F_X(H)=\left\{\sum_t\frac{\|X^{(t)}h^{(t)}\|_2^2}{n_t}\right\}^{1/2},
 \qquad b=\min\{p,\lceil M\bar k\rceil\},
\]
where \(M\) is a sufficiently large constant depending on \(K\).
Bernstein's inequality for Gaussian squares, a \(1/4\)-net, and a union
bound over supports and tasks give
\begin{equation}
\label{eq:row-cone-sparse-covariance}
 c\|v\|_2^2\le v^\top\widehat\Sigma_t v\le C\|v\|_2^2,
 \qquad t\le T,\quad\|v\|_0\le b.
\end{equation}
The logarithm of the number of support-net points is at most
\(C\{b\log(ep/b)+\log T\}\), which is covered by the sample-size condition
\eqref{eq:row-cone-curvature-sample}, up to a constant depending on \(K\).

\emph{Part (i).}
If \(b=p\), summing \eqref{eq:row-cone-sparse-covariance} over tasks
proves the claim.  Otherwise, sort the predictor-vector norms of \(H\) and divide them
into consecutive blocks \(J_0,J_1,\ldots\), each of size \(b\), except
possibly the last.  The sorted-block inequality gives
\[
 \sum_{\ell\ge1}\|H_{J_\ell}\|_F
 \le\frac{\|H\|_{2,1}}{\sqrt b}
 \le\frac K{\sqrt M}\|H\|_F.
\]
For each block, \eqref{eq:row-cone-sparse-covariance} implies
\(c\|H_{J_\ell}\|_F\le F_X(H_{J_\ell})\le C\|H_{J_\ell}\|_F\).
Using the triangle inequality and
\(\|H_{J_0}\|_F\ge\|H\|_F-\sum_{\ell\ge1}\|H_{J_\ell}\|_F\), we obtain
\[
\begin{aligned}
 F_X(H)
 &\ge c\|H_{J_0}\|_F-C\sum_{\ell\ge1}\|H_{J_\ell}\|_F
 \ge(c-C'K/\sqrt M)\|H\|_F,\\
 F_X(H)
 &\le C\sum_{\ell\ge0}\|H_{J_\ell}\|_F
 \le C(1+K/\sqrt M)\|H\|_F.
\end{aligned}
\]
Choosing \(M\) sufficiently large, then squaring and dividing by \(2T\),
proves part (i).

\emph{Part (ii): first allow only \(b\) predictors to vary.}
We begin with
\begin{equation}
\label{eq:fusion-mixed-model-curvature}
 c\|H\|_F^2\le F_X(H)^2\le C\|H\|_F^2
\end{equation}
for all matrices whose predictors outside a set \(J\), \(|J|\le b\), are
constant across tasks.  If \(b=p\), this again follows from
\eqref{eq:row-cone-sparse-covariance}.  Otherwise enlarge \(J\) to size
\(b\).  Gaussian regression gives
\[
 X_{J^c}^{(t)}=X_J^{(t)}\Gamma_t+Z_t,\qquad
 \Gamma_t=\Sigma_{t,JJ}^{-1}\Sigma_{t,JJ^c},
\]
where \(Z_t\) is independent of \(X_J^{(t)}\) and has independent
Gaussian rows with covariance
\[
 \Omega_t=\Sigma_{t,J^cJ^c}
 -\Sigma_{t,J^cJ}\Sigma_{t,JJ}^{-1}\Sigma_{t,JJ^c},
 \qquad c_0I\preceq\Omega_t\preceq C_0I.
\]
Let \(P_t\) project onto the orthogonal complement of the columns of
\(X_J^{(t)}\).  The per-task sample condition ensures \(n_t\ge2b\)
and full column rank, uniformly over \(J,t\).  Conditional on the
\(X_J^{(t)}\), an orthogonal change of basis gives
\[
 X_{J^c}^{(t)\top}P_tX_{J^c}^{(t)}
 \ \overset{d}{=}\ \sum_{i=1}^{n_t-b}z_{it}z_{it}^\top,
 \qquad z_{it}\sim N(0,\Omega_t),
\]
independently across \(i,t\).  It follows that
\begin{equation}
\label{eq:fusion-residual-pooled-design}
 \frac1T\sum_t\frac{X_{J^c}^{(t)\top}P_tX_{J^c}^{(t)}}{n_t}
 \succeq cI_{p-b}
 \qquad\text{for every }J.
\end{equation}
For the probability bound, a fixed unit direction has conditional mean at
least \(c_0/2\).  The weights of its centered Gaussian squares are
\(1/(Tn_t)\); their squared sum and maximum are at most
\(1/(Tn_{\min})\).  Thus a fixed constant deviation has probability at
most \(2e^{-cn_{\min}T}\).  The net contributes \(C(p-b)\) to the logarithm, and the predictor sets
contribute at most \(b\log(ep/b)\).  The sample-size assumptions cover both
terms.  Integrating the conditional bound for each fixed \(J\) and then taking
the union bound proves \eqref{eq:fusion-residual-pooled-design}; independence
across different choices of \(J\) is not required.

On the same event, sparse and pooled covariance concentration give
\[
 cI\preceq X_J^{(t)\top}X_J^{(t)}/n_t\preceq CI
 \quad\text{for every }J,t,\qquad
 cI\preceq\widehat\Sigma_\bullet\preceq CI.
\]
For a matrix in the specified model, write
\(h_{J^c}^{(t)}=v\) and \(h_J^{(t)}=u_t\).  Projecting out the
columns of \(X_J^{(t)}\) gives
\begin{equation}
\label{eq:fusion-mixed-control-v}
 T\|v\|_2^2
 \lesssim\sum_t\frac{\|P_tX_{J^c}^{(t)}v\|_2^2}{n_t}
 \le F_X(H)^2.
\end{equation}
Also \(X_J^{(t)}u_t=X^{(t)}h^{(t)}-X_{J^c}^{(t)}v\).  The sparse
lower covariance bound and the pooled upper bound imply
\[
 \left(\sum_t\|u_t\|_2^2\right)^{1/2}
 \lesssim F_X(H)
 +\left(\sum_t\frac{\|X_{J^c}^{(t)}v\|_2^2}{n_t}\right)^{1/2}
 \lesssim F_X(H)+\sqrt T\|v\|_2
 \lesssim F_X(H).
\]
Since \(\|H\|_F^2=T\|v\|_2^2+\sum_t\|u_t\|_2^2\), this proves the
lower bound in \eqref{eq:fusion-mixed-model-curvature}.  The upper bound
follows from
\[
 F_X(H)\le
 \left(\sum_t\frac{\|X_J^{(t)}u_t\|_2^2}{n_t}\right)^{1/2}
 +\left(\sum_t\frac{\|X_{J^c}^{(t)}v\|_2^2}{n_t}\right)^{1/2}
 \lesssim\left(\sum_t\|u_t\|_2^2\right)^{1/2}+\sqrt T\|v\|_2.
\]

\emph{Part (ii): extend to the stated predictor bound.}
Sort the predictor-vector norms of \(G\) into blocks \(J_0,J_1,\ldots\) of size
\(b\).  As in part (i),
\[
 \sum_{\ell\ge1}\|G_{J_\ell}\|_F\le\frac K{\sqrt M}\|G\|_F.
\]
Every block is centered, so
\(\|a\mathbf1_T^\top+G_{J_0}\|_F^2
=T\|a\|_2^2+\|G_{J_0}\|_F^2\).
Apply \eqref{eq:fusion-mixed-model-curvature} to this matrix and to each
tail block.  The triangle inequality gives
\[
\begin{aligned}
 F_X(a\mathbf1_T^\top+G)
 &\ge c\{T\|a\|_2^2+\|G_{J_0}\|_F^2\}^{1/2}
       -C\sum_{\ell\ge1}\|G_{J_\ell}\|_F\\
 &\ge(c-C'K/\sqrt M)\{T\|a\|_2^2+\|G\|_F^2\}^{1/2}.
\end{aligned}
\]
The triangle inequality gives the matching upper bound.  Choose \(M\)
sufficiently large.  Since the covariance events are uniform over supports,
the predictor ordering may depend on the observed matrix.  Allocating fixed
fractions of the failure probability completes the proof.
\end{proof}

\begin{proof}[Proof of Lemma~\ref{lem:upper-noise-bounds}]
We first derive a conditional Gaussian bound common to both estimators.  We
then use it once for Pairwise Fusion and twice when the group penalty is also
present.

\paragraph{Conditional noise bounds.}
On the sparse covariance event, for every matrix \(V\) supported on an
eligible set of at most \(b\) predictors,
\[
 \operatorname{Var}\{\langle Z,V\rangle\mid X\}
 =\frac1{T^2}\sum_t\frac{\sigma_t^2}{n_t}
             v^{(t)\top}\widehat\Sigma_t v^{(t)}
 \le\frac{C\sigma^2}{T^2n_{\min}}\|V\|_F^2.
\]
No covariance matrix needs to be invertible for this calculation.
The sample conditions give this event with
\(b=\min\{p,r\}\) for part (i) and \(b=s\) for part (ii).

For part (i), the pooled covariance is also bounded above and below.
To obtain the same variance bound after fitting the common vector, stack
\[
 U=\begin{bmatrix}X^{(1)}/\sqrt{Tn_1}\\ \vdots\\
                  X^{(T)}/\sqrt{Tn_T}\end{bmatrix},\qquad
 \varepsilon=\begin{bmatrix}e^{(1)}/\sqrt{Tn_1}\\ \vdots\\
                  e^{(T)}/\sqrt{Tn_T}\end{bmatrix},
 \qquad \Pi=U(U^\top U)^{-1}U^\top.
\]
Then \(U^\top U=\widehat\Sigma_\bullet\),
\(U^\top\varepsilon=\bar w\), and
\(\operatorname{Cov}(\varepsilon\mid X)
\preceq\sigma^2(Tn_{\min})^{-1}I\).
For a matrix \(V\), let \(\mathcal V(V)\) have task block
\(X^{(t)}v^{(t)}/\sqrt{Tn_t}\).  Direct calculation gives
\begin{equation}
\label{eq:projection-score-identities}
 \langle Z,V\rangle=\varepsilon^\top\mathcal V(V),\qquad
 \langle\zeta,V\rangle=\varepsilon^\top(I-\Pi)\mathcal V(V).
\end{equation}
Since \(I-\Pi\) is an orthogonal projection,
\[
 \operatorname{Var}\{\langle\zeta,V\rangle\mid X\}
 \le\frac{\sigma^2}{Tn_{\min}}\|\mathcal V(V)\|_2^2
 =\frac{\sigma^2}{T^2n_{\min}}
       \sum_t v^{(t)\top}\widehat\Sigma_t v^{(t)}.
\]
Thus the same variance bound holds for both noise matrices on the stated
predictor supports.  We use only the upper bound on the noise covariance; when
task sample sizes or noise variances differ, the projection need not reduce the
actual covariance.

We will also use the following two consequences in the shared-task proof:
\begin{equation}
\label{eq:pooled-projection-hn-tail}
 \Pr\!\left\{\|h_N\|_2^2>
 C\frac{\sigma^2(p+u)}{n_{\min}T}\,\middle|\,X\right\}
 \le2e^{-u},\qquad u\ge0,
\end{equation}
on \(cI\preceq\widehat\Sigma_\bullet\preceq CI\), and
\begin{equation}
\label{eq:pooled-projection-coordinate-residual}
 \max_{j,t}\|(w_t-\widehat\Sigma_t h_N)_j\|_{\psi_2\mid X}
 \le C\frac\sigma{\sqrt{n_{\min}}}
\end{equation}
if \(\max_{j,t}(\widehat\Sigma_t)_{jj}\le C\).
The first follows from
\(\operatorname{Cov}(h_N\mid X)
\preceq\sigma^2(Tn_{\min})^{-1}\widehat\Sigma_\bullet^{-1}\)
and a Gaussian norm bound.  For the second, take \(V\) with one nonzero entry in
\eqref{eq:projection-score-identities} and use
\(\|\mathcal V(V)\|_2^2=(\widehat\Sigma_t)_{jj}/T\), and multiply
by \(T\).  At fixed \(u\) depending on \(\delta\),
\eqref{eq:pooled-projection-hn-tail} proves
\eqref{eq:fusion-pooled-noise-rate}.

\paragraph{Uniform bounds on a fixed number of coordinates.}
For this proof only, write
\(\tau=C\sigma/(T\sqrt{n_{\min}})\).
The preceding conditional Gaussian bounds imply
\(\|\langle W,V\rangle\|_{\psi_2\mid X}\le\tau\|V\|_F\)
when \(W\) is the relevant noise matrix and \(V\) is supported on at
most \(b\) predictors.  For fixed integers \(k\le b\), \(m\le kT\), a
\(1/2\)-net and a union bound give, with conditional probability at least
\(1-2e^{-u}\),
\begin{equation}
\label{eq:supportwise-gaussian-net-bound}
 \sup_{\substack{|J|\le k,\ A\subseteq J\times[T]\\ |A|\le m}}
 \|W_A\|_F
 \le C\tau\sqrt{k\log(ep/k)+m\log(ekT/m)+u}.
\end{equation}
After enlarging the predictor and coordinate sets if necessary, there are at
most \(\binom pk\binom{kT}m\) choices, each requiring at most \(5^m\) net
points.

Let \(P=\mathbf1_T\mathbf1_T^\top/T\).  Right multiplication by \(P\)
or \(I-P\) preserves predictor support and does not increase Frobenius norm,
so the same argument applies to \(WP\) and \(W(I-P)\).  The constant
predictors have dimension \(k\), giving the sharper bound
\begin{equation}
\label{eq:common-score-support-bound}
 \sup_{|J|\le k}\|(WP)_J\|_F
 \le C\tau\sqrt{k\log(ep/k)+u}.
\end{equation}
We use this bound only at the fixed support sizes specified below.  Taking
\(u=\log(C/\delta)\) for a sufficiently large constant \(C\), its contribution
can be absorbed into the constants because \(\delta\) is fixed.

\paragraph{Part (i): absorb small coordinates by fusion.}
Apply \eqref{eq:supportwise-gaussian-net-bound} to \(W=\zeta\),
\(k=\min\{p,r\}\), and \(m=r\).  The expression under the square
root is at most a constant times \(r\log(epT/r)\).  Hence
\[
 \sup_{|A|\le r}\|\zeta_A\|_F
 \le C\tau\sqrt{r\log(epT/r)}.
\]
Every set of \(r\) entries involves at most \(\min\{p,r\}\) predictors,
which justifies the unrestricted coordinate supremum in this display.
Choose \(\theta=A_\delta\tau\sqrt{\log(epT/r)}\), with
\(A_\delta\) sufficiently large, and define the entrywise operation
\[
 S_\theta(x)=\operatorname{sign}(x)(|x|-\theta)_+.
\]
Fewer than \(r\) entries of \(\zeta\) can exceed \(\theta\) in
magnitude; otherwise, selecting any \(r\) of them would contradict the
preceding bound.  Consequently, on that same event,
\[
 \|S_\theta(\zeta)\|_F
 \le C\tau\sqrt{r\log(epT/r)},\qquad
 \|\zeta-S_\theta(\zeta)\|_\infty\le\theta.
\]
For a centered \(G\), \eqref{eq:pairwise-centering-l1} therefore
gives
\[
\begin{aligned}
 |\langle\zeta,G\rangle|
 &\le\theta\|G\|_1+
      C\tau\sqrt{r\log(epT/r)}\,\|G\|_F\\
 &\le\frac{2\theta}{T}\|GD^\top\|_1+
      C\tau\sqrt{r\log(epT/r)}\,\|G\|_F.
\end{aligned}
\]
The stated tuning, with a sufficiently large constant, has
\(\lambda\ge16\theta/T\) and
\(\tau\sqrt{r\log(epT/r)}\lesssim\lambda T\sqrt r\).
This proves \eqref{eq:fusion-profiled-score}.

\paragraph{Part (ii): also absorb small predictor scores by the group penalty.}
Use \(W=Z(I-P)\), and write, only within this argument,
\[
 \ell=\log(esT/r),\qquad
 d=s\log(ep/s)+r\ell,\qquad
 \theta=A_\delta\tau\sqrt\ell.
\]
After thresholding the small centered coordinates, the remainder is
\[
 R=ZP+S_\theta(W)(I-P).
\]
Since \(Z-R=\{W-S_\theta(W)\}(I-P)\), the same centering inequality
yields, for every \(H\),
\begin{equation}
\label{eq:small-score-fusion-bound}
 |\langle Z-R,H\rangle|
 \le\theta\|H(I-P)\|_1
 \le\frac{2\theta}{T}\|HD^\top\|_1.
\end{equation}
We next bound \(R\) on any \(s\) predictors.  Set
\[
 M=\min\{sT,\lceil d/\ell\rceil\}.
\]
This fixed integer satisfies \(r\le M\le sT\),
\(M\le2d/\ell\), and
\[
 s\log(ep/s)+M\log(esT/M)\le3d.
\]
Apply \eqref{eq:supportwise-gaussian-net-bound} to \(W\) at
\((k,m)=(s,M)\), and \eqref{eq:common-score-support-bound} to \(Z\)
at \(k=s\).  Uniformly over \(|J|\le s\), these two events give
\[
 \|(ZP)_J\|_F\le C\tau\sqrt{s\log(ep/s)},\qquad
 \sup_{\substack{A\subseteq J\times[T]\\|A|\le M}}
       \|W_A\|_F\le C\tau\sqrt d.
\]
If \(M<sT\), then \(M\ell\ge d\).  For sufficiently large
\(A_\delta\), the preceding bound rules out \(M\) entries of \(W_J\) above
\(\theta\).
Thus the nonzero entries of \(S_\theta(W)_J\) are covered by the same
coordinate bound.  If \(M=sT\), all entries are already covered.
In either case, projection does not increase the predictor-restricted norm, so
\begin{equation}
\label{eq:two-level-row-subset-remainder}
 \sup_{|J|\le s}\|R_J\|_F\le C\tau\sqrt d.
\end{equation}

Choose \(\gamma=A'_\delta\tau\sqrt{d/s}\).  If \(A'_\delta\) is large
enough, \eqref{eq:two-level-row-subset-remainder} implies that fewer than
\(s\) predictor blocks of \(R\) have norm above \(\gamma\).  Control those
blocks jointly by Cauchy--Schwarz and the displayed bound, and control every
remaining block by \(\gamma\).  This gives
\[
 |\langle R,H\rangle|
 \le\gamma\|H\|_{2,1}+C\tau\sqrt d\,\|H\|_F.
\]
Combine this with \eqref{eq:small-score-fusion-bound}.  The stated tuning
has \(\lambda\ge16\theta/T\), \(\nu\ge8\gamma\), and
\(\tau\sqrt d\lesssim\nu\sqrt s\), proving
\eqref{eq:localized-group-fusion-score}.  Allocating fixed fractions of \(\delta\) to the design and conditional
noise events gives the stated probability.  These events are already uniform
over all supports of the required sizes, including supports selected from the
realized noise.
\end{proof}

\section{PROOF OF THE MINIMAX LOWER BOUND}
\label{sec:lower-bound-proof}

The lower bound has two sources: the active predictors are unknown, and so
are the exception locations within those predictors.  We bound each difficulty
separately and then combine the two rates.  The theorem proof comes first,
followed by the packing details.

\subsection{Proof of Theorem~\ref{thm:minimax-lower-bounds}}

The next lemma constructs the exception sets used in the proof.  Its proof,
together with the elementary packing facts it uses, appears in
Section~\ref{subsec:packing-details}.

\begin{lemma}[Packing predictor-specific exception locations]
\label{lem:predictor-exception-packing}
Let \(s\ge1\), \(T\ge2\), and \(1\le r\le s(T-1)\).  Put
\[
        r_0=\min\{r,s\lfloor T/2\rfloor\}.
\]
There is a collection \(\mathcal A\) of subsets of
\(\{1,\ldots,s\}\times\{1,\ldots,T\}\) such that every \(A\in\mathcal A\)
satisfies
\[
        |A|\le r_0\le r,
        \qquad
        |A_j|\le\lfloor T/2\rfloor,
        \qquad
        A_j:=\{t:(j,t)\in A\},
\]
and, for universal constants \(c_1,c_2>0\),
\begin{align}
 |A\triangle A'|&\ge c_1r_0,
 \qquad A\ne A',
 \label{eq:predictor-exception-separation}\\
 \log|\mathcal A|
 &\ge c_2r_0\log\!\left(\frac{esT}{r_0}\right).
 \label{eq:predictor-exception-cardinality}
\end{align}
Moreover,
\begin{equation}
\label{eq:r0-r-comparison}
 r_0\log\!\left(\frac{esT}{r_0}\right)
 \ge c_1 r\log\!\left(\frac{esT}{r}\right).
\end{equation}
\end{lemma}

\begin{proof}[Proof of Theorem~\ref{thm:minimax-lower-bounds}]
Under the balanced Gaussian model,
\begin{equation}
\label{eq:gaussian-regression-kl}
 D_{\rm KL}(P_B\|P_{B'})
 =\frac{n}{2\sigma^2}\|B-B'\|_F^2,
\end{equation}
because the design distribution is common to all parameters and
\(\mathbb E X^{(t)\top}X^{(t)}=nI_p\).  We use standard Fano and Assouad
arguments.  For a packing with at
least eight points, squared separation \(\rho^2\), and maximum KL divergence
from a fixed point bounded by a sufficiently small multiple of its
log-cardinality, Fano's inequality gives expected squared-error risk at least
\(c\rho^2\).  For a packing of between two and seven points, choose two
points and decrease the same universal amplitude constant so that their KL
divergence is at most \(1/8\).  Pinsker's inequality bounds their total
variation distance by \(1/4\).  Assigning an estimate to its nearer packing
point then gives maximal squared-error risk at least
\(\rho^2(1-1/4)/8=3\rho^2/32\).  Hence the same lower-bound conclusion also covers packings with only a few
points.

\paragraph{Step 1: Unknown active predictors.}
First fix the active predictors to be \(\{1,\ldots,s\}\) and let
\[
        u=\mathbf1_T/\sqrt T.
\]
For \(\omega\in\{-1,1\}^s\), put
\[
        B_\omega
        =a\sum_{j=1}^s\omega_j e_j u^\top,
        \qquad
        a^2=\frac{\sigma^2}{16n}.
\]
Each active predictor is constant across tasks, so
\(B_\omega\in\Theta_{\rm share}(s,r)\) for every \(r\ge0\).  Neighboring hypercube vertices have squared Frobenius
distance \(4a^2\), and their KL divergence is \(1/8\) by
\eqref{eq:gaussian-regression-kl}.  Assouad's argument therefore gives
\begin{equation}
\label{eq:lower-bound-active-values}
 \inf_{\widehat B}\sup_{B^*\in\Theta_{\rm share}(s,r)}
 \mathbb E\|\widehat B-B^*\|_F^2
 \ge c\frac{\sigma^2s}{n}.
\end{equation}

If \(s\le p/2\), apply Lemma~\ref{lem:constant-weight-packing}(i) to
\(s\)-subsets \(J\subseteq[p]\).  For each packing set define
\[
        B_J=\mu\sum_{j\in J}e_j u^\top.
\]
These matrices also have no exceptions: \(r(B_J)=0\).  Pairwise distances satisfy
\[
        \|B_J-B_{J'}\|_F^2
        =\mu^2|J\triangle J'|
        \ge c\mu^2s,
\]
and the divergence from a fixed packing point is at most
\(ns\mu^2/\sigma^2\).  Taking
\[
        \mu^2
        =c_\mu\frac{\sigma^2}{n}\log(p/s)
\]
with sufficiently small universal \(c_\mu\), Fano's inequality gives
\begin{equation}
\label{eq:lower-bound-active-support}
 \inf_{\widehat B}\sup_{B^*\in\Theta_{\rm share}(s,r)}
 \mathbb E\|\widehat B-B^*\|_F^2
 \ge c\frac{\sigma^2}{n}s\log(p/s).
\end{equation}
When \(s>p/2\), \(\log(ep/s)\) is bounded by an absolute constant, so
\eqref{eq:lower-bound-active-values} already has the required order.
Combining \eqref{eq:lower-bound-active-values} and
\eqref{eq:lower-bound-active-support} yields
\begin{equation}
\label{eq:lower-bound-active-total}
 \inf_{\widehat B}\sup_{B^*\in\Theta_{\rm share}(s,r)}
 \mathbb E\|\widehat B-B^*\|_F^2
 \ge c\frac{\sigma^2}{n}s\log(ep/s).
\end{equation}

\paragraph{Step 2: Unknown exception locations.}
Let \(\mathcal A\) be the collection in
Lemma~\ref{lem:predictor-exception-packing}.  Fix a constant \(b>0\), and for
\(A\in\mathcal A\) define
\[
 (B_A)_{j,t}
 =\begin{cases}
 b+\mu\mathbf1_{\{(j,t)\in A\}},&j\le s,\\
 0,&j>s.
 \end{cases}
\]
For predictor \(j\le s\), let \(w_j=|A_j|\).  Since
\(w_j\le\lfloor T/2\rfloor\), the value \(b\) is a most common coefficient and
\[
 \min_{a\in\mathbb R}
 \|(B_A)_{j,:}-a\mathbf1_T\|_0=w_j.
\]
Thus every packing point belongs to \(\Theta_{\rm share}(s,r)\).
For two packing elements,
\[
        \|B_A-B_{A'}\|_F^2
        =\mu^2|A\triangle A'|.
\]
Let \(r_0\) be as in Lemma~\ref{lem:predictor-exception-packing} and take
\[
        \mu^2
        =c_\mu\frac{\sigma^2}{n}
        \log\!\left(\frac{esT}{r_0}\right)
\]
with a sufficiently small universal constant.  The maximum KL divergence
from a fixed packing point is at most
\[
 C\frac{n\mu^2r_0}{\sigma^2}
 \le c' r_0\log(esT/r_0),
\]
which is a sufficiently small multiple of \(\log|\mathcal A|\) by
\eqref{eq:predictor-exception-cardinality}.  Fano's inequality and
\eqref{eq:predictor-exception-separation} therefore give
\[
 \inf_{\widehat B}\sup_{B^*\in\Theta_{\rm share}(s,r)}
 \mathbb E\|\widehat B-B^*\|_F^2
 \ge
 c\frac{\sigma^2}{n}
 r_0\log\!\left(\frac{esT}{r_0}\right).
\]
Using \eqref{eq:r0-r-comparison},
\begin{equation}
\label{eq:lower-bound-predictor-exceptions}
 \inf_{\widehat B}\sup_{B^*\in\Theta_{\rm share}(s,r)}
 \mathbb E\|\widehat B-B^*\|_F^2
 \ge
 c\frac{\sigma^2}{n}
 r\log\!\left(\frac{esT}{r}\right).
\end{equation}
\paragraph{Step 3: Combine the two bounds.}
The minimax risk is bounded below by both
\eqref{eq:lower-bound-active-total} and
\eqref{eq:lower-bound-predictor-exceptions}.  Since
\(\max\{a,b\}\ge(a+b)/2\), the two bounds combine, after changing the
universal constant, to give \eqref{eq:predictor-sharing-lower-bound}.
\end{proof}

\subsection{Packing Details}
\label{subsec:packing-details}

We use two elementary packing facts.  They are stated here so that the
lower-bound construction, including the small-parameter cases, is self-contained.

\begin{lemma}[Two elementary packing facts]
\label{lem:constant-weight-packing}
\label{lem:finite-alphabet-packing}
\emph{(i) Sets of equal size.}
For integers \(1\le k\le N/2\), there is a collection \(\mathcal V\)
of \(k\)-subsets of \(\{1,\ldots,N\}\) such that
\[
 \log|\mathcal V|\ge c k\log(N/k),\qquad
 |A\triangle A'|\ge k/2\quad(A\ne A').
\]
In particular, \(|\mathcal V|\ge2\).

\emph{(ii) Words in a finite alphabet.}
For integers \(d\ge1\) and \(M\ge2\), there is
\(\mathcal W\subseteq\{1,\ldots,M\}^d\) such that
\[
 \log|\mathcal W|\ge c d\log M,\qquad
 d_{\rm H}(w,w')\ge c d\quad(w\ne w'),
\]
where \(d_{\rm H}\) counts the differing positions.  Constants are universal.
\end{lemma}

\begin{proof}
For part (i), apply the constant-weight Varshamov--Gilbert construction,
greedily selecting \(k\)-subsets separated by at least \(k/2\) in symmetric
difference.  The Hamming-ball count gives
\(\log|\mathcal V|\ge c k\log(N/k)\), including \(k=1\).

For part (ii), greedily select words and, after each selection, remove the
Hamming ball of radius \(\lfloor d/8\rfloor\).  Such a ball has size at most
\[
 \sum_{k\le d/8}\binom dk(M-1)^k
 \le \exp\{dH(1/8)\}M^{d/8},
\]
where \(H\) is the binary entropy function.  For \(M\ge2\), its logarithm
is at most \((1-c)d\log M\) for a universal \(c>0\).  It follows that the selected family has log-cardinality at least
\(cd\log M\), while every pair of words is separated by more than \(d/8\).
Reducing \(c\) if necessary covers the integer endpoints.
\end{proof}

\begin{proof}[Proof of Lemma~\ref{lem:predictor-exception-packing}]
We use a direct product construction.  First suppose \(r_0\le s\) and set
\(k=r_0\).  If \(k\le s/2\), use
Lemma~\ref{lem:constant-weight-packing}(i) to choose a family of \(k\)-subsets
of the predictors with log-cardinality at least \(ck\log(s/k)\) and predictor-set
symmetric difference at least \(k/2\).  For each selected predictor place exactly
one exception.  Lemma~\ref{lem:finite-alphabet-packing}(ii), with alphabet size \(T\), assigns
the task locations with Hamming separation at least \(ck\) and
log-cardinality at least \(ck\log T\).  Combining the predictor-set and task-location codes gives
\[
 \log|\mathcal A|\ge
 ck\{\log(s/k)+\log T\}
 \ge ck\log(e sT/k),
\]
with pairwise symmetric difference at least \(ck\).  If \(k>s/2\), fix the
first \(k\) predictors and use only the task-location code from
Lemma~\ref{lem:finite-alphabet-packing}(ii).  Since \(s/k<2\) and \(T\ge2\),
its log-cardinality \(ck\log T\) is still at least
a constant multiple of \(k\log(e sT/k)\).

Next suppose \(r_0>s\), and set
\[
        q=\lfloor r_0/s\rfloor,
        \qquad r_1=sq.
\]
Then \(1\le q\le\lfloor T/2\rfloor\) and
\(r_0/2\le r_1\le r_0\).  By
Lemma~\ref{lem:constant-weight-packing}(i), there is a family \(\mathcal V\) of
\(q\)-subsets of the \(T\) tasks with pairwise symmetric difference at least
\(q/2\) and
\[
        \log|\mathcal V|\ge cq\log(T/q).
\]
Apply Lemma~\ref{lem:finite-alphabet-packing}(ii) with alphabet \(\mathcal V\)
to obtain a subset of \(\mathcal V^s\) in which any two elements differ on at
least a constant fraction of the \(s\) predictors and whose log-cardinality is at
least \(cs\log|\mathcal V|\).  Interpreting each codeword as the exception set in
each predictor gives sets with exactly \(r_1\le r\) entries, at most \(T/2\) per
predictor, pairwise symmetric difference at least \(cr_1\), and
\[
 \log|\mathcal A|
 \ge cr_1\log(T/q)
 \ge cr_0\log(e sT/r_0),
\]
where we used \(r_1\in[r_0/2,r_0]\) and \(r_0\le sT/2\).
This proves \eqref{eq:predictor-exception-separation}--
\eqref{eq:predictor-exception-cardinality}.

Finally, compare \(r_0\) with \(r\).  They are equal when
\(r\le s\lfloor T/2\rfloor\).  Otherwise,
\(r_0=s\lfloor T/2\rfloor\ge sT/3\), since \(T\ge2\).
The function \(x\mapsto x\log(esT/x)\) is increasing on \((0,sT]\)
and is at most \(sT\).  Therefore
\[
 r_0\log(esT/r_0)\ge r_0\ge sT/3
       \ge\tfrac13 r\log(esT/r),
\]
which proves \eqref{eq:r0-r-comparison}.
\end{proof}

\section{PROOF OF THE SHARED-TASK RESULT}
\label{sec:taskwise-proof}

We first impose the constraint that all tasks in \(\mathcal C\) share one
coefficient vector and bound the error of this constrained estimator.  We then
verify the optimality conditions of the original problem.  Strict inequalities
for pairs within \(\mathcal C\) force every unrestricted minimizer to agree on
that set, so the constrained error bounds carry over to the original estimator.

The set \(\mathcal C\), the true parameter, and the tuning parameter are
fixed before sampling; the proof does not search over \(\mathcal C\).
Write
\[
 m=|\mathcal C|=T-q,\qquad
 S_\Delta=\bigcup_{u\in\mathcal C^c}\operatorname{supp}(\delta^{*(u)}),
 \qquad |S_\Delta|=s_\Delta,
\]
and, for a general failure probability \(\delta\), let
\[
 L=\log\{pT(T-1)/\delta\},\qquad
 \widehat\Sigma_{\mathcal C}=\frac1m
          \sum_{t\in\mathcal C}\widehat\Sigma_t.
\]
We retain \(\widehat\Sigma_\bullet,h_N,\zeta\) from
\eqref{eq:common-noise-and-residual}.  The theorem proof will set
\(\delta=(pT)^{-2}\).

\subsection{Three Ingredients}
\label{subsec:taskwise-proof-ingredients}

The first lemma collects the design and noise bounds.  The second controls
the profiled quadratic loss, and the third removes the temporary equality
constraint.  Their proofs appear in Section~\ref{subsec:taskwise-supporting-proofs}.

\begin{lemma}[A joint design and noise bound]
\label{lem:taskwise-design-noise}
\label{lem:profiled-design-event}
\label{lem:profiled-score-pooled-noise}
Under Assumption~\ref{ass:task-specific-design}, suppose \(m\ge T/2\),
\(n_{\min}\ge CL\), and
\(n_{\min}T\ge C\{p+\log(T/\delta)\}\).
With probability at least \(1-\delta\),
\begin{equation}
\label{eq:profiled-design-event}
 cI\preceq\widehat\Sigma_\bullet\preceq CI,\qquad
 cI\preceq\widehat\Sigma_{\mathcal C}\preceq CI,\qquad
 \max_{t,j}(\widehat\Sigma_t)_{jj}\le C,
\end{equation}
\begin{equation}
\label{eq:taskwise-operator-design-event}
 \max_t\|\widehat\Sigma_t\|_{\rm op}
 \le C(1+p/n_{\min}),
\end{equation}
and
\begin{align}
 \|h_N\|_2^2
 &\le C\frac{\sigma^2\{p+\log(2/\delta)\}}{n_{\min}T},
 \label{eq:pooled-noise-common-rate}\\
 \max_t\|w_t-\widehat\Sigma_t h_N\|_\infty
 &\le C\sigma\sqrt{L/n_{\min}}.
 \label{eq:profiled-cluster-noise-residual}
\end{align}
In particular, the last display divided by \(T\) bounds
\(\max_t\|\zeta_t\|_\infty\).
\end{lemma}

\paragraph{The loss and penalty under the equality constraint.}
For outside-task deviations \(G=[g^{(u)}:u\in\mathcal C^c]\), the complete
pairwise penalty becomes
\begin{equation}
\label{eq:collapsed-pairwise-penalty}
 \mathcal P(G)=m\sum_{u\in\mathcal C^c}\|g^{(u)}\|_1
 +\sum_{\substack{u<v\\u,v\in\mathcal C^c}}
             \|g^{(u)}-g^{(v)}\|_1.
\end{equation}
The first term collects all pairs involving the \(m\) shared tasks.  The
penalty is additive across predictors and satisfies
\begin{equation}
\label{eq:collapsed-penalty-l1-equivalence}
 m\|G\|_1\le\mathcal P(G)\le T\|G\|_1.
\end{equation}
For the upper bound, each outside coefficient appears in \(q-1\) additional
pairs, so the triangle inequality gives the factor \(m+q-1\le T\).  When \(q=0\), \(G\) is empty and the penalty is zero.

After fitting the common vector, the quadratic loss for a deviation error
\(G\) is
\begin{equation}
\label{eq:profiled-quadratic-variational}
 \mathcal Q_X(G)=\min_{a\in\mathbb R^p}\frac1{2T}
 \left\{\sum_{t\in\mathcal C}a^\top\widehat\Sigma_ta
 +\sum_{u\in\mathcal C^c}(a+g^{(u)})^\top
                   \widehat\Sigma_u(a+g^{(u)})\right\}.
\end{equation}
Because at least half the tasks are shared, this profiled loss still controls
\(G\), as shown next.

\begin{lemma}[Loss after fitting the common vector]
\label{lem:profiled-pairwise-curvature}
Under Assumption~\ref{ass:task-specific-design}, suppose \(m\ge T/2\) and
\begin{equation}
\label{eq:profiled-curvature-sample}
 n_{\min}\ge C(1+s_\Delta q)L,\qquad
 n_{\min}T\ge C\{p+s_\Delta qL+\log(T/\delta)\}.
\end{equation}
With probability at least \(1-\delta\), every \(G\) satisfying
\begin{equation}
\label{eq:profiled-entry-cone}
 \|G_{S_\Delta^c}\|_1\le6\|G_{S_\Delta}\|_1
\end{equation}
obeys
\begin{align}
 \mathcal Q_X(G)&\ge\frac cT\|G\|_F^2,
 \label{eq:profiled-curvature-lower}\\
 \frac1T\sum_{u\in\mathcal C^c}
       g^{(u)\top}\widehat\Sigma_u g^{(u)}
 &\le\frac CT\|G\|_F^2,
 \label{eq:profiled-curvature-upper}\\
 \left\|\widehat\Sigma_\bullet^{-1}\frac1T
       \sum_{u\in\mathcal C^c}\widehat\Sigma_u g^{(u)}\right\|_2^2
 &\le\frac CT\min\!\left\{1,
       \frac qT\left(1+\frac p{n_{\min}}\right)\right\}\|G\|_F^2.
 \label{eq:profiled-cross-bound}
\end{align}
If \(s_\Delta q=0\), condition \eqref{eq:profiled-entry-cone} forces
\(G=0\), and the conclusions hold on the pooled design event.
\end{lemma}

\begin{lemma}[A residual condition that forces equality]
\label{lem:complete-cluster-fusion}
Suppose \(\lambda>0\) and \(m\ge2\).  Let \(\widetilde B\) minimize
\eqref{eq:pairwise-fusion-obj} subject to equality of its task vectors on
\(\mathcal C\).  Write
\[
 r_t=\frac{X^{(t)\top}\{y^{(t)}-X^{(t)}\widetilde\beta^{(t)}\}}{n_t},
 \qquad \bar r=\frac1m\sum_{t\in\mathcal C}r_t.
\]
If
\begin{equation}
\label{eq:complete-cluster-fusion-condition}
 \max_{t\in\mathcal C}\|r_t-\bar r\|_\infty<\frac{T\lambda m}{2},
\end{equation}
then \(\widetilde B\) is an unrestricted minimizer, and every unrestricted
minimizer is constant on \(\mathcal C\).
\end{lemma}

\subsection{Proof of Theorem~\ref{thm:shared-task-pooling}}
\label{subsec:taskwise-theorem-proof}

\begin{proof}
Set \(\delta=(pT)^{-2}\), so
\(L\asymp\log(epT)\) and
\(p+\log(2/\delta)\asymp p+\log T\).
The theorem's per-task condition also implies
\(n_{\min}T\gtrsim s_\Delta qL+\log(T/\delta)\).
Hence the assumptions of the preceding design and noise lemmas are satisfied.
Apply those lemmas with fixed fractions of \(\delta\) and work on the resulting
joint event, which has probability at least \(1-\delta\).

Both the constrained and unrestricted objectives attain their minima.  A
bounded objective controls all task differences because \(\lambda>0\), and
the loss then controls the common component because
\(\widehat\Sigma_\bullet\succ0\).  The sublevel sets are therefore bounded,
including on the closed subspace that imposes equality on \(\mathcal C\).

\paragraph{Step 1: Fit the common vector for fixed deviations.}
Let \(\widetilde B\) be any constrained minimizer, and write its error as
\[
 \widetilde\beta^{(t)}-\beta^{*(t)}=h\quad(t\in\mathcal C),\qquad
 \widetilde\beta^{(u)}-\beta^{*(u)}=h+g^{(u)}\quad(u\in\mathcal C^c).
\]
Put \(G=[g^{(u)}:u\in\mathcal C^c]\) and
\(G^*=[\delta^{*(u)}:u\in\mathcal C^c]\).
For fixed \(G\), minimizing the loss over \(h\) gives
\begin{equation}
\label{eq:profiled-common-error-decomposition}
 h=h_N+h_G,\qquad
 h_G=-\widehat\Sigma_\bullet^{-1}\frac1T
               \sum_{u\in\mathcal C^c}\widehat\Sigma_u g^{(u)}.
\end{equation}
As in the dense upper-bound proof, completing the square leaves the profiled
objective
\[
 \mathcal Q_X(G)-\sum_{u\in\mathcal C^c}\zeta_u^\top g^{(u)}
       +\lambda\mathcal P(G^*+G),
\]
up to a constant independent of \(G\).  Compare this with \(G=0\), which
keeps the true deviations and refits only the common vector, to obtain
\begin{equation}
\label{eq:profiled-pairwise-basic}
 \mathcal Q_X(G)+\lambda\{\mathcal P(G^*+G)-\mathcal P(G^*)\}
 \le\sum_{u\in\mathcal C^c}\zeta_u^\top g^{(u)}.
\end{equation}

\paragraph{Step 2: Bound the constrained estimation error.}
By \eqref{eq:profiled-cluster-noise-residual},
\eqref{eq:collapsed-penalty-l1-equivalence}, and \(m\ge T/2\), a
sufficiently large fixed \(C_\lambda\) ensures
\[
 \left|\sum_{u\in\mathcal C^c}\zeta_u^\top g^{(u)}\right|
 \le\frac\lambda{16}\mathcal P(G).
\]
The true deviations vanish outside \(S_\Delta\).  Additivity across
predictors and the triangle inequality therefore give
\[
 \mathcal P(G^*+G)-\mathcal P(G^*)
 \ge\mathcal P(G_{S_\Delta^c})-\mathcal P(G_{S_\Delta}).
\]
Substitution into \eqref{eq:profiled-pairwise-basic} yields
\[
 \mathcal Q_X(G)+\frac{15\lambda}{16}\mathcal P(G_{S_\Delta^c})
 \le\frac{17\lambda}{16}\mathcal P(G_{S_\Delta}).
\]
Using \(m\|G\|_1\le\mathcal P(G)\le T\|G\|_1\), and the fact that
\(G_{S_\Delta}\) has at most \(s_\Delta q\) entries, we obtain
\begin{equation}
\label{eq:profiled-pairwise-cone-and-bound}
 \|G_{S_\Delta^c}\|_1\le6\|G_{S_\Delta}\|_1,\qquad
 \mathcal Q_X(G)\lesssim\lambda T\sqrt{s_\Delta q}\,\|G\|_F.
\end{equation}
If \(s_\Delta q=0\), these inequalities force \(G=0\).  Otherwise,
Lemma~\ref{lem:profiled-pairwise-curvature} gives
\begin{equation}
\label{eq:profiled-deviation-rate}
 \|G\|_F^2\lesssim\lambda^2T^4s_\Delta q.
\end{equation}
Its bound on the common-vector adjustment also gives
\begin{equation}
\label{eq:profiled-bias-common-rate}
 \|h_G\|_2^2\lesssim\lambda^2T^3s_\Delta q
       \min\!\left\{1,\frac qT\left(1+\frac p{n_{\min}}\right)\right\}.
\end{equation}
The constants in these two bounds are independent of \(C_\lambda\).

Using \(h=h_N+h_G\), the noise bound, and the chosen \(\lambda\), the
second term inside the minimum gives
\[
 \|h\|_2^2\lesssim
 \frac{\sigma^2(p+\log T)}{n_{\min}T}
 +\frac{\sigma^2s_\Delta q^2L}{n_{\min}T^2}
 +\frac{\sigma^2ps_\Delta q^2L}{n_{\min}^2T^2}.
\]
The last term is absorbed into \(\sigma^2p/(n_{\min}T)\), because
\((s_\Delta qL/n_{\min})(q/T)\lesssim1\).  Thus every constrained
minimizer satisfies
\begin{equation}
\label{eq:restricted-pairwise-shared-task-rate}
 \|h\|_2^2\lesssim\frac{\sigma^2}{n_{\min}T}
       \left\{p+\log T+\frac{s_\Delta q^2}{T}L\right\}.
\end{equation}

\paragraph{Step 3: Remove the equality constraint.}
If \(m=1\), the constraint is vacuous.  Otherwise, for
\(t\in\mathcal C\), the constrained residual is
\(r_t=w_t-\widehat\Sigma_t(h_N+h_G)\).  Subtracting its average over
\(\mathcal C\) gives
\begin{equation}
\label{eq:profiled-residual-decomposition}
 r_t-\bar r
 =T\left(\zeta_t-\frac1m\sum_{v\in\mathcal C}\zeta_v\right)
       -(\widehat\Sigma_t-\widehat\Sigma_{\mathcal C})h_G.
\end{equation}
By \eqref{eq:profiled-cluster-noise-residual} and the triangle inequality,
the first term is uniformly bounded by
\(C\sigma\sqrt{L/n_{\min}}\).
For the second term, the diagonal and operator bounds imply
\[
 \|\widehat\Sigma_t e_j\|_2^2
 \le\|\widehat\Sigma_t\|_{\rm op}(\widehat\Sigma_t)_{jj}
 \lesssim1+p/n_{\min}.
\]
Since \(\|\widehat\Sigma_{\mathcal C}\|_{\rm op}\lesssim1\), it
follows that
\[
 \max_{t\in\mathcal C}
 \|(\widehat\Sigma_t-\widehat\Sigma_{\mathcal C})h_G\|_\infty
 \lesssim\sqrt{1+p/n_{\min}}\,\|h_G\|_2.
\]
This deterministic bound remains valid for the fitted, response-dependent
\(h_G\).  Combining it with \eqref{eq:profiled-bias-common-rate}, and
using \(m\asymp T\), yields
\begin{equation}
\label{eq:taskwise-fusion-ratio-outside}
 \frac{\max_{t\in\mathcal C}\|r_t-\bar r\|_\infty}{T\lambda m}
 \lesssim C_\lambda^{-1/2}
       +\frac{q\sqrt{s_\Delta}}{T}
       +\frac{pq\sqrt{s_\Delta}}{n_{\min}T}.
\end{equation}
The first term follows from
\(T\lambda m\asymp\sqrt{C_\lambda}\sigma\sqrt{L/n_{\min}}\).
The remaining two terms are small under the stated lower bounds on \(T\)
and \(n_{\min}T\).  Choosing those constants first and then taking
\(C_\lambda\) sufficiently large makes the ratio smaller than \(1/2\).
No similarity between the task covariance matrices is required.

Lemma~\ref{lem:complete-cluster-fusion} therefore implies that every
unrestricted minimizer is constant on \(\mathcal C\).  Such a minimizer also
solves the constrained problem, so
\eqref{eq:restricted-pairwise-shared-task-rate} applies.  Substituting
\(L\asymp\log(epT)\) gives the shared-task bound.

\paragraph{Step 4: Bound the outside tasks.}
For \(q\ge1\), the outside-task error is \(h+g^{(u)}\), so
\[
 \frac1q\sum_{u\in\mathcal C^c}\|h+g^{(u)}\|_2^2
 \le2\|h\|_2^2+\frac2q\|G\|_F^2.
\]
Using \eqref{eq:profiled-deviation-rate} and
\eqref{eq:restricted-pairwise-shared-task-rate} gives
\[
 \frac1q\sum_{u\in\mathcal C^c}\|h+g^{(u)}\|_2^2
 \lesssim\frac{\sigma^2}{n_{\min}}
 \left\{\frac{p+\log T}{T}+s_\Delta L
                +\frac{s_\Delta q^2L}{T^2}\right\}.
\]
Since \(q\le T/2\), the last term is absorbed by \(s_\Delta L\), giving
\eqref{eq:outside-task-average-error}.  Finally,
\[
 \max_{u\in\mathcal C^c}\|h+g^{(u)}\|_2^2
 \le2\|h\|_2^2+2\|G\|_F^2
\]
gives the maximum-error bound in Remark~\ref{rem:taskwise-interpretation}.
\end{proof}

\subsection{Proofs of the Three Ingredients}
\label{subsec:taskwise-supporting-proofs}

\begin{proof}[Proof of Lemma~\ref{lem:taskwise-design-noise}]
For a fixed unit vector, squared Gaussian projections are uniformly
sub-exponential.  Applying \eqref{eq:weighted-bernstein} with weights
\(1/(Tn_t)\), followed by a \(1/4\)-net argument, gives the required pooled
covariance bounds.  The pooled sample-size condition therefore bounds
\(\widehat\Sigma_\bullet\) above and below.  The same argument applies to
\(\widehat\Sigma_{\mathcal C}\) because \(m\ge T/2\).  A union bound over the \(pT\) individual coordinate
variances gives \(\max_{t,j}(\widehat\Sigma_t)_{jj}\le C\) when
\(n_{\min}\gtrsim L\).

For the task-specific operator norms, the same Gaussian-square and net
argument, combined with a union bound over tasks, gives
\[
 \max_t\|\widehat\Sigma_t-\Sigma_t\|_{\rm op}
 \lesssim\sqrt{\frac{p+\log(T/\delta)}{n_{\min}}}
             +\frac{p+\log(T/\delta)}{n_{\min}}.
\]
Together with \(n_{\min}\gtrsim L\), this proves
\eqref{eq:taskwise-operator-design-event}; it does not require
\(n_{\min}\gtrsim p\).

On these design events, \eqref{eq:pooled-projection-hn-tail} gives
\eqref{eq:pooled-noise-common-rate}.  The conditional coordinate bound
\eqref{eq:pooled-projection-coordinate-residual} and a union bound over
\((j,t)\) give \eqref{eq:profiled-cluster-noise-residual}.
Allocating fixed fractions of \(\delta\) to the design and conditional noise
events gives the stated joint probability.
\end{proof}

\begin{proof}[Proof of Lemma~\ref{lem:profiled-pairwise-curvature}]
On the pooled design event, the shared tasks contribute a fixed fraction of
the total covariance.  Hence, for some numerical \(\kappa>0\),
\[
 \frac1T\sum_{u\in\mathcal C^c}\widehat\Sigma_u
 \preceq(1-\kappa)\widehat\Sigma_\bullet.
\]
For any \(v\), Cauchy--Schwarz over the outside tasks gives
\[
\begin{aligned}
 \left|v^\top\frac1T\sum_{u\in\mathcal C^c}
                    \widehat\Sigma_u g^{(u)}\right|^2
 &\le\left(\frac1T\sum_{u\in\mathcal C^c}
                    v^\top\widehat\Sigma_u v\right)
       \left(\frac1T\sum_{u\in\mathcal C^c}
                    g^{(u)\top}\widehat\Sigma_u g^{(u)}\right)\\
 &\le(1-\kappa)(v^\top\widehat\Sigma_\bullet v)
       \left(\frac1T\sum_{u\in\mathcal C^c}
                    g^{(u)\top}\widehat\Sigma_u g^{(u)}\right).
\end{aligned}
\]
Take \(v=\widehat\Sigma_\bullet^{-1}T^{-1}
\sum_{u\in\mathcal C^c}\widehat\Sigma_u g^{(u)}\).
Completing the square in \eqref{eq:profiled-quadratic-variational} shows that
profiling out the common vector can remove at most a fraction \(1-\kappa\)
of the outside-task quadratic loss.  Hence
\[
 \mathcal Q_X(G)\ge\frac\kappa{2T}
       \sum_{u\in\mathcal C^c}g^{(u)\top}\widehat\Sigma_u g^{(u)}.
\]

Condition \eqref{eq:profiled-entry-cone} implies
\[
 \|G\|_{2,1}\le\|G\|_1\le7\|G_{S_\Delta}\|_1
 \le7\sqrt{s_\Delta q}\,\|G\|_F.
\]
Also \(\|G\|_{2,1}\le\sqrt p\,\|G\|_F\).
Pad \(G\) with zero columns on \(\mathcal C\) and apply part (i) of
Lemma~\ref{lem:upper-design-bounds} with
\(k=\min\{p,s_\Delta q\}\).  The stated sample condition implies its
sample requirement, so
\[
 \frac1T\sum_{u\in\mathcal C^c}
       g^{(u)\top}\widehat\Sigma_u g^{(u)}
 \asymp\frac1T\|G\|_F^2.
\]
Together, these bounds give \eqref{eq:profiled-curvature-lower} and
\eqref{eq:profiled-curvature-upper}.

The same Cauchy--Schwarz inequality and
\(\widehat\Sigma_\bullet\succeq cI\) bound the squared norm in
\eqref{eq:profiled-cross-bound} by \(C\|G\|_F^2/T\).
For the other bound, use \(A^2\preceq\|A\|_{\rm op}A\) for positive
semidefinite \(A\), the individual operator bound, and the preceding
quadratic upper bound:
\[
\begin{aligned}
 \left\|\widehat\Sigma_\bullet^{-1}\frac1T
       \sum_{u\in\mathcal C^c}\widehat\Sigma_u g^{(u)}\right\|_2^2
 &\lesssim\frac q{T^2}
       \sum_{u\in\mathcal C^c}\|\widehat\Sigma_u g^{(u)}\|_2^2\\
 &\lesssim\frac q{T^2}\left(1+\frac p{n_{\min}}\right)
       \sum_{u\in\mathcal C^c}g^{(u)\top}\widehat\Sigma_u g^{(u)}\\
 &\lesssim\frac q{T^2}\left(1+\frac p{n_{\min}}\right)\|G\|_F^2.
\end{aligned}
\]
Taking the smaller of the two estimates gives
\eqref{eq:profiled-cross-bound}.  The argument is uniform over all \(G\)
satisfying \eqref{eq:profiled-entry-cone}.  If \(s_\Delta q=0\), the
condition forces \(G=0\), and no sparse design bound is needed.
\end{proof}

\begin{proof}[Proof of Lemma~\ref{lem:complete-cluster-fusion}]
We construct subgradients for the unrestricted problem and use a strict
subgradient inequality to rule out unequal coefficients within
\(\mathcal C\).
Write \(\operatorname{Sign}(x)=\{\operatorname{sign}(x)\}\) for
\(x\ne0\), and \(\operatorname{Sign}(0)=[-1,1]\).

\paragraph{Construct the internal-pair subgradients.}
In predictor \(j\), let \(c_j\) be the common constrained coefficient on
\(\mathcal C\).  The \(m\) pairs involving an outside task \(u\) contribute
\(\lambda m|c_j-\widetilde b_{j,u}|\).
Choose
\(s_{j,u}\in\operatorname{Sign}(c_j-\widetilde b_{j,u})\) from the constrained
optimality conditions and assign the same value to each corresponding pair:
\(s_{j,tu}=s_{j,u}\) for \(t\in\mathcal C\).
The common-coordinate equation is
\[
 \frac1T\sum_{t\in\mathcal C}r_{j,t}
 =\lambda m\sum_{u\in\mathcal C^c}s_{j,u}.
\]
Dividing by \(m\) gives
\(\bar r_j/T=\lambda\sum_{u\in\mathcal C^c}s_{j,tu}\).
For an outside task, the collapsed subgradient
\(-\lambda m s_{j,u}\) equals the sum of the \(m\) copied pairwise
subgradients, so its optimality equation is unchanged.

For an internal pair, choose
\[
 s_{j,tu}=\frac{r_{j,t}-r_{j,u}}{T\lambda m},
 \qquad t,u\in\mathcal C,\quad t\ne u.
\]
These values are antisymmetric, and
\[
 \sum_{u\in\mathcal C\setminus\{t\}}s_{j,tu}
 =\frac{r_{j,t}-\bar r_j}{T\lambda},\qquad
 |s_{j,tu}|\le
 \frac{2\max_{v\in\mathcal C}\|r_v-\bar r\|_\infty}{T\lambda m}<1.
\]
Retain the constrained subgradients on pairs between outside tasks and
use antisymmetry on the reverse orientations.  The full equations are now
\[
 \frac{r_{j,t}}T=\lambda\sum_{u\ne t}s_{j,tu},
 \qquad s_{j,tu}\in
 \operatorname{Sign}(\widetilde b_{j,t}-\widetilde b_{j,u}).
\]
These are exactly the subgradient optimality conditions for the unrestricted
convex objective, so \(\widetilde B\) is also an unrestricted minimizer.

\paragraph{Show that every minimizer agrees on \(\mathcal C\).}
Let \(F\) denote the unrestricted objective and compare an arbitrary
\(B\) with \(\widetilde B\).  Expanding the squared loss and using the
optimality equations cancels the linear term.  The remaining pairwise terms
are nonnegative by the subgradient inequality.  For an internal pair, where the
constrained coefficients agree, that term is at least
\((1-|s_{j,tu}|)|b_{j,t}-b_{j,u}|\).  Therefore
\begin{equation}
\label{eq:cluster-strict-gap}
\begin{aligned}
 F(B)-F(\widetilde B)
 &\ge\frac1{2T}\sum_t
 \frac{\|X^{(t)}(\beta^{(t)}-\widetilde\beta^{(t)})\|_2^2}{n_t}\\
 &\quad+\lambda\sum_j\sum_{\substack{t<u\\t,u\in\mathcal C}}
       (1-|s_{j,tu}|)|b_{j,t}-b_{j,u}|.
\end{aligned}
\end{equation}
If \(B\) is another minimizer, the left side is zero.  Since every term on
the right is nonnegative and every internal weight \(1-|s_{j,tu}|\) is
strictly positive, all coefficient differences within \(\mathcal C\) must
vanish.
\end{proof}

\section{ADDITIONAL SIMULATIONS AND EXPERIMENTAL DETAILS}
\label{sec:simulation-details}

The main-text random-departure experiment varies predictor-specific sharing
and reports total squared Frobenius error \(\|\widehat B-B^*\|_F^2\).
The supplementary shared-task experiment reports the
maximum coefficient error over \(\mathcal C\).  Each reported setting uses
30 independent repetitions, and competing methods are evaluated on matched
data; parentheses contain Monte Carlo standard errors.

Tuning parameters are selected on independent pilot data and then held fixed
throughout the reporting repetitions.  These experiments are finite-sample
comparisons of estimation behavior.  They are not designed to verify the
sufficient sample-size conditions of the high-probability theorems.

\subsection{Random Departures on the Active Support}
\label{subsec:random-departure-setup}

\paragraph{Coefficients and data.}
Use \((p,s,n,T)=(40,40,140,60)\) in low dimension and
\((180,12,80,60)\) in high dimension.  Draw the active set
\(S\subseteq\{1,\ldots,p\}\) uniformly among sets of size \(s\).
For every \(j\in S\), independently draw a uniform sign \(c_j\), a
uniform permutation \(\pi_j\) of the task indices, and independent
\(z_{j,t}\sim N(0,1)\).  At departure fraction
\(\alpha\in\{0,0.1,0.2,0.4,0.6,0.8,1\}\), let
\(M_j(\alpha)=\{\pi_j(1),\ldots,\pi_j(\alpha T)\}\) and set
\[
 \widetilde b_{j,t}=
 \begin{cases}
 c_j+2z_{j,t}\mathbf1_{\{t\in M_j(\alpha)\}},&j\in S,\\
 0,&j\notin S.
 \end{cases}
\]
Every active predictor has exactly \(\alpha T\) perturbed entries, selected
independently across predictors.  Normalize each active predictor to norm \(\sqrt T\)
to obtain \(U\), and set
\[
 B^*=aU,\qquad
 a=\left\{\frac{4\sigma^2}
 {T^{-1}\sum_t U_{:,t}^{\top}\Sigma U_{:,t}}\right\}^{1/2},
 \qquad \sigma=1.2,\quad \Sigma_{jk}=0.5^{|j-k|}.
\]
This fixes average population SNR at \(4\).  The Gaussian departures have
standard deviation \(2\) before normalization; their final magnitudes change
with the predictor-specific and SNR scales.  The unperturbed coefficients for
a predictor remain exactly equal, and all perturbed values are distinct almost
surely.  Hence
\(q_j=\min\{\alpha T,T-1\}\) on \(S\) and
\(r=s\min\{\alpha T,T-1\}\).  At \(100\%\) departures the exception count is
\(s(T-1)\), because with \(T\) distinct coefficients the largest equality
group has size one.
Design rows are independent \(N_p(0,\Sigma)\), and errors are independent
\(N(0,\sigma^2)\), independently of all designs.

Within a dimension regime and repetition, the active set, signs, permutations,
Gaussian departures, designs, and noise are shared across all seven fractions.
The masks are nested as \(\alpha\) increases; predictor-wise normalization
and the SNR scale are recomputed at each fraction.  These random objects are regenerated
independently for each reporting repetition.  Pilot data are independent of
all reporting data and are likewise paired across fractions within a regime.

\paragraph{Tuning and reporting.}
Each setting uses one pilot dataset and five matched folds.  All methods use
the same training and validation observations; penalties minimize the
average validation prediction error over folds and tasks.  Separate and Pooled
Ridge use \(0\) and 19 logarithmically spaced penalties from \(10^{-4}\) to
\(10^2\).  Write
\[
 \ell=1+\log\{epT^2(T-1)/2\},\qquad
 \lambda_0=\frac{\sigma\sqrt{2\ell}}{T^2\sqrt n},\qquad
 \nu_0=\frac{\sigma\sqrt\ell}{T\sqrt n}.
\]
Fusion uses multipliers \(\{0,0.03,0.1,0.3,1,3\}\) of \(\lambda_0\).
Group penalties use \(\{0.1,0.3,1,3,10\}\) times \(\nu_0\).
Sparse Pairwise Fusion searches their Cartesian product, including all
Group Lasso candidates at zero fusion.  Separate and Pooled Lasso use
multipliers \(\{0.1,0.3,1,3,10\}\) of
\(\sigma\sqrt{\log(p)/n}\) and
\(\sigma\sqrt{\log(p)/(nT)}\), respectively.  These scales use the full
pilot sample size and remain numerically fixed within the CV folds and the
reporting fits.  All regressions omit an intercept.

For every method and fraction, the selected penalties are held fixed for
30 reporting repetitions.  We compute \(\|\widehat B-B^*\|_F^2\) over the
entire matrix, including inactive predictors, then report its mean and sample
standard deviation divided by \(\sqrt{30}\).  No repetitions are discarded.
Fused fits must pass the solver's convergence and objective checks; group
and fused reporting fits also pass a normalized proximal optimality-residual
check with tolerance \(2\times10^{-5}\).  Truth matrices, departure masks,
fitted coefficients, pilot scores, and seeds are retained for independent
verification.

\subsection{Error on the Shared Tasks}
\label{subsec:taskwise-settings}

Table~\ref{tab:taskwise-results} reports
\[
 E_{\max}=\max_{t\in\mathcal C}
 \|\widehat\beta^{(t)}-\beta_{\rm in}^*\|_2^2.
\]
The maximum is taken within each repetition and then averaged over 30
independent repetitions; the parenthesized value is the Monte Carlo standard
error of these 30 maxima.

\paragraph{Two-cluster construction.}
Every setting has \(p=40\), \(n=30\), \(T=60\), and \(\sigma=1\).
Design rows are independent
\(N_p(0,\Sigma)\), where \(\Sigma_{jk}=0.3^{|j-k|}\), and the errors are
independent \(N(0,1)\).  We draw a dense random-sign vector
\(\beta_{\rm in}^*\) and scale it so that
\((\beta_{\rm in}^*)^\top\Sigma\beta_{\rm in}^*=4\).  A second independent
random-sign direction \(v\), normalized to \(\|v\|_2=1\), defines
\[
 \beta_{\rm out}^*=\beta_{\rm in}^*+\Delta v.
\]
For each \(q\), we draw \(\mathcal C\) once, uniformly among subsets of size
\(T-q\), and keep it fixed across the pilot and all reporting repetitions.  All tasks in
\(\mathcal C\) use \(\beta_{\rm in}^*\), and all outside tasks use the same
\(\beta_{\rm out}^*\).  Hence every predictor has one most common value and \(q\)
exceptions, and the two coefficient vectors are exactly distance \(\Delta\)
apart.  Panel (a) uses \(\Delta=4\) and \(q\in\{10,15,20,25\}\); panel (b)
uses \(q=15\) and \(\Delta\in\{0.5,1,2,4,6,8,12\}\).  Within each paired
pilot or reporting repetition in panel (b), \(\beta_{\rm in}^*\), \(v\), the
designs, and the noise are shared across all values of \(\Delta\), so only
\(\Delta\) changes.  Between repetitions, we regenerate \(\beta_{\rm in}^*\),
\(v\), the designs, and the noise, but keep \(\mathcal C\) fixed.

\paragraph{Methods and tuning.}
Separate Ridge fits each task independently with one shared ridge parameter.
Pooled-all Ridge estimates one vector from all tasks.  Pairwise Fusion uses
the unrestricted pairwise-fusion objective.  None of the three methods receives
\(\mathcal C\) or either coefficient vector.

Each of the ten distinct \((q,\Delta)\) settings uses one pilot
data set for five-fold cross-validation.  Each fold contains six observations
per task, with the same row split used by every method and candidate.  Ridge
parameters range over 11 logarithmically spaced values from \(10^{-4}\) to
\(10\), and the fused multiplier ranges over
\(\{0.1,0.3,0.5,0.7,1,3,10\}\).  Pooled-all Ridge and Pairwise Fusion minimize
validation prediction error separately in every setting.  Separate Ridge
selects \(\alpha=0.316\) at the common reference setting
\((q,\Delta)=(15,4)\) and uses it throughout; its definition therefore does
not change as \(\Delta\) changes.  The selected parameters are held fixed for
30 fresh reporting repetitions.  Pilot and reporting seeds are disjoint.

\begin{table}[!t]
\centering
\small
\setlength{\tabcolsep}{7pt}
\caption{Maximum squared coefficient error over \(t\in\mathcal C\) for \(p=40\), \(n=30\), \(T=60\), and \(\sigma=1\). Here \(q=|\mathcal C^c|\) and \(\Delta=\|\beta^*\sb {\rm out}-\beta^*\sb {\rm in}\|\sb 2\). Entries are Monte Carlo means (standard errors) over 30 repetitions. No method is given \(\mathcal C\).}
\label{tab:taskwise-results}
\textbf{(a) Number of outside tasks, \(\Delta=4\)}\par\smallskip
\begin{tabular}{@{}cccc@{}}
\toprule
\(q\) & Separate Ridge & Pooled-all Ridge & Pairwise Fusion \\
\midrule
10 & \(2.9143\;(0.0499)\) & \(0.5180\;(0.0148)\) & \(0.4138\;(0.0148)\) \\
15 & \(2.9946\;(0.0668)\) & \(1.0582\;(0.0232)\) & \(0.5684\;(0.0186)\) \\
20 & \(3.1482\;(0.0633)\) & \(1.7775\;(0.0324)\) & \(0.9478\;(0.0209)\) \\
25 & \(3.0561\;(0.0675)\) & \(2.7443\;(0.0464)\) & \(1.7167\;(0.0441)\) \\
\bottomrule
\end{tabular}
\par\medskip
\textbf{(b) Distance between the two coefficient vectors, \(q=15\)}\par\smallskip
\begin{tabular}{@{}cccc@{}}
\toprule
\(\Delta\) & Separate Ridge & Pooled-all Ridge & Pairwise Fusion \\
\midrule
0.5 & \(2.9946\;(0.0668)\) & \(0.0399\;(0.0016)\) & \(0.0421\;(0.0019)\) \\
1 & \(2.9946\;(0.0668)\) & \(0.0884\;(0.0026)\) & \(0.0892\;(0.0026)\) \\
2 & \(2.9946\;(0.0668)\) & \(0.2802\;(0.0067)\) & \(0.4408\;(0.0167)\) \\
4 & \(2.9946\;(0.0668)\) & \(1.0582\;(0.0232)\) & \(0.5684\;(0.0186)\) \\
6 & \(2.9946\;(0.0668)\) & \(2.2396\;(0.0502)\) & \(1.6274\;(0.0398)\) \\
8 & \(2.9946\;(0.0668)\) & \(3.9695\;(0.0884)\) & \(1.6950\;(0.0421)\) \\
12 & \(2.9946\;(0.0668)\) & \(7.6859\;(0.1775)\) & \(1.7632\;(0.0450)\) \\
\bottomrule
\end{tabular}
\end{table}

In panel (a), increasing \(q\) leaves fewer tasks in \(\mathcal C\) and creates
more between-group pairwise differences.  The Pairwise Fusion error consequently rises from
\(0.414\) at \(q=10\) to \(1.717\) at \(q=25\).  It remains below both alternative methods throughout; Pooled-all deteriorates faster because its
single coefficient vector averages two well-separated groups.

Panel (b) shows why separation matters.  When \(\Delta\leq1\), treating the
two groups as one is competitive: at \(\Delta=0.5\), Pooled-all Ridge and Pairwise
Fusion have errors \(0.040\) and \(0.042\).  Pooled-all remains better through
\(\Delta=2\), but its merging bias then grows rapidly.  Pairwise Fusion is already
better at \(\Delta=4\), and the gap widens at larger separations: at
\(\Delta=12\), their errors are \(1.763\) and \(7.686\), respectively.  Thus
the benefit of learning separate clusters appears once their coefficient
vectors are sufficiently far apart.  The Separate Ridge entries are identical
by construction: the paired experiment changes only the outside-task vector,
whereas Separate Ridge is fitted task by task and its error is evaluated only
on the unchanged tasks in \(\mathcal C\).

The two task groups differ on all \(p\) coordinates, and \(q\) is
deliberately nontrivial.  These settings therefore probe a difficult
finite-sample regime rather than attempting to isolate the asymptotic rate
in Theorem~\ref{thm:shared-task-pooling}.

\section{RECS ANALYSIS DETAILS}
\label{sec:recs-details}

\paragraph{Sample and predictors.}
We use the January 2024 version 7 of the 2020 RECS public-use file
\citep{eia2020recs}.  Each of its 18,496 records is a distinct household
identifier.  All 50 states and DC are retained; task sizes range from 143
in Delaware to 1,152 in California.  The response is the natural logarithm
of \texttt{KWH}, total annual electricity use including self-generated solar
power.  Required fields are available for every record.  We retain the
released imputed values and do not trim outcomes, filter state-specific
contrasts, or remove states.

The fixed additive design has 16 columns.  Three indicate electric main
space heating (\texttt{FUELHEAT}=5), electric water heating
(\texttt{FUELH2O}=5), and air-conditioning use (\texttt{AIRCOND}=1).
The electric-heating reference includes both other main heating fuels and
households that did not use space heating; the not-applicable fuel code is
not treated as a numerical covariate.  Four quantitative columns are
log energy-consuming floor area (\texttt{TOTSQFT\_EN}), household size
(\texttt{NHSLDMEM}), and heating and cooling degree days
(\texttt{HDD65}, \texttt{CDD65}), each degree-day variable divided by 1,000.

Housing type contributes four indicators: mobile homes, attached houses,
apartments in buildings with two to four units, and apartments in buildings
with at least five units, relative to detached houses.  Household income
is grouped into below \$25,000, \$25,000--49,999, \$50,000--99,999, and
at least \$100,000, with the first group as reference.  Construction period
is grouped into before 1980, 1980--1999, and 2000 onward, with the first
period as reference.  Income and construction period thus contribute
three and two indicators.  These encodings are fixed before fitting;
no interactions or further predictors are added.

\paragraph{Weighted loss and common scaling.}
Let \(w_{it}\) be the released \texttt{NWEIGHT} and let
\(a_{it}=w_{it}/\sum_{i\in I_t}w_{it}\), where \(I_t\) is the current
training subset of state \(t\).  On commonly standardized predictors and
response, we minimize
\begin{equation}
\label{eq:recs-objective}
 \frac{1}{2T}\sum_{t=1}^T\sum_{i\in I_t}a_{it}
       (y_{it}-\alpha_t-x_{it}^{\top}\beta^{(t)})^2
 +\lambda\sum_{j=1}^{16}\sum_{t<u}|\beta_j^{(t)}-\beta_j^{(u)}|.
\end{equation}
There are \(\binom{51}{2}=1,275\) equally weighted pairwise differences per
predictor.  The intercepts \(\alpha_t\) are unpenalized; there is no group,
lasso, or ridge penalty.  Census regions and state adjacency do not enter
the estimator.

For each fitting subset, predictor and response means and standard deviations
are computed under the mixture that assigns mass \(1/T\) to each state
and weights \(a_{it}\) within it.  The same scale is used for every state;
we do not standardize separately by state.  Weighted state means are then
profiled out to eliminate the intercepts.  Validation households do not
contribute to these transformations.  The final fit uses transformations
estimated from the full sample; exported raw coefficients and intercepts
recover predictions in original log-kWh units.

\paragraph{Ordinary five-fold cross-validation.}
A seeded permutation (seed 123) assigns households to five approximately
equal folds within each state.  For each candidate parameter and each fold,
we train on the other four folds and validate on that fold.  The validation
score is the average across states of the within-state survey-weighted
log-response squared error.  Its mean over the five validation folds selects
one parameter, after which all households are used to refit one final
\(\widehat B\).  The five fitted matrices are not averaged.

We parameterize the candidate grid by
\(\rho=\lambda/\lambda_{\rm full}\), where \(\lambda_{\rm full}\)
is the smallest penalty admitting the shared-slope weighted least-squares
solution, computed from each training subset.  The grid is
\[
 \{0,.001,.002,.003,.005,.007,.01,.015,.02,.03,.05,.07,.1,.15,.2,.3,.5,.7,1\}.
\]
At zero we use statewise weighted least squares, taking the minimum-norm
solution if needed.  At one, all slopes are shared but intercepts remain
state specific.  Minimizing mean validation error selects \(\rho=0.2\),
with mean validation MSE 0.214651.  The full-data threshold is
\(\lambda_{\rm full}=1.2210626564\times10^{-4}\), giving
\(\lambda=2.4421253128\times10^{-5}\) for the reported fit.
We use neither a one-standard-error rule nor a separate test split for
this coefficient analysis; the validation score is a tuning criterion,
not an independent estimate of test error.

\paragraph{Numerical checks and equality groups.}
The implementation uses the exact pairwise-fusion proximal map with
accelerated proximal gradients and a feasible dual certificate for
\(\nu=0\).  All validation candidates have relative primal--dual gaps
at most \(10^{-8}\); the final fit has gap below \(5\times10^{-11}\).
An independent explicit-pair ADMM solution at the selected parameter differs
by less than \(10^{-8}\) in standardized coefficients and
\(4\times10^{-14}\) in objective value.

For each standardized predictor coefficient vector, sorted values form groups of
maximum diameter \(10^{-5}\).  The mean of the largest group defines the
reference value, and departures within that group are displayed as zero.
Tied largest groups are resolved by proximity to the predictor-specific median and then
the lower mean; electric water heating is the only tied predictor and is marked
with an asterisk.  This convention chooses a display reference, not a
unique common effect.  The full matrix has 247 numerical exceptions from the largest equality groups
out of 816 entries.  No predictor is completely fused, although 14 predictors have a
majority group and the median largest group size is 40.

\paragraph{Complete heatmap and interpretation.}
The heatmap retains every predictor and all 51 tasks, ordered alphabetically
within Census regions.  Every cell receives color
\(s_j(\widehat\beta^{\rm raw}_{jt}-\widehat a_j)\), where \(s_j\)
is the common full-data predictor standard deviation defined above.  No
predictor selection, state selection, color clipping, or stability masking is
applied.  A cross denotes fewer than ten observations in one binary
category, or a predictor that is constant within that state.  This is an
annotation rather than a preprocessing rule.  Both the complete raw and
standardized coefficient matrices accompany the figure.

The full-sample centered design is rank deficient in DC and Hawaii.
The corresponding coefficients can borrow information through fusion even
when a local contrast is absent.  For air conditioning, Delaware, Iowa,
Nebraska, and Oklahoma have fewer than ten households without the equipment.
Survey weighting, imputation, the 2020 observation period, and limited
within-state variation all constrain interpretation.  We make no causal
claim, do not report survey-design confidence intervals, and do not equate
numerical fusion with verified equality of population coefficients.

\end{document}